\documentclass[10pt]{article} 
\usepackage[preprint]{tmlr}

\usepackage{amsmath,amsfonts,bm}

\def\eqref#1{equation~\ref{#1}}

\def\1{\bm{1}}

\DeclareMathAlphabet{\mathsfit}{\encodingdefault}{\sfdefault}{m}{sl}
\SetMathAlphabet{\mathsfit}{bold}{\encodingdefault}{\sfdefault}{bx}{n}

\usepackage{hyperref}
\usepackage{url}
\usepackage{amsmath,amsfonts,bm}
\usepackage{amssymb}
\usepackage{amsthm}
\usepackage{cleveref}
\usepackage{booktabs}

\newtheorem{theorem}{Theorem}
\newtheorem{definition}{Definition}
\newtheorem{lemma}{Lemma}
\newtheorem{proposition}{Proposition}
\newtheorem{corollary}{Corollary}
\newtheorem{remark}{Remark}

\def\eqref#1{equation~\ref{#1}}

\def\1{\bm{1}}

\DeclareMathAlphabet{\mathsfit}{\encodingdefault}{\sfdefault}{m}{sl}
\SetMathAlphabet{\mathsfit}{bold}{\encodingdefault}{\sfdefault}{bx}{n}

\title{PAC-Bayes Beyond Parameter Space: Behavioral Equivalence, Z-Information, and Exact Complexity Decomposition}

\author{\name Vasant G Honavar \email vuh14@psu.edu \\
      \addr Artificial Intelligence Research Laboratory \\ Department of Informatics and Intelligent Systems\\
      The Pennsylvania State University
      \AND
      \name Satish Kumar Keshri \email skk6485@psu.edu \\
     \addr Artificial Intelligence Research Laboratory \\ Department of Informatics and Intelligent Systems\\
      The Pennsylvania State University
      \AND
      \name Neil Ashtekar \email  nca5096@psu.edu\\
      \addr Artificial Intelligence Research Laboratory \\ Department of Computer Science and Engineering\\
      The Pennsylvania State University
      \AND
      \name Zehao Liu \email  zml5418@psu.edu\\
      \addr Artificial Intelligence Research Laboratory \\ Department of Informatics and Intelligent Systems\\
      The Pennsylvania State University
     }

\def\month{MM}  
\def\year{YYYY} 
\def\openreview{\url{https://openreview.net/forum?id=XXXX}} 

\begin{document}

\maketitle

\begin{abstract}
PAC-Bayes theory provides generalization guarantees by controlling the
Kullback--Leibler (KL) divergence between posterior and prior distributions
over a chosen hypothesis representation. However, predictive risk depends only
on the predictive behavior induced by a hypothesis, not on the particular
internal realization that implements that behavior. In modern
over-parameterized learning systems, many distinct configurations may induce
identical predictive behavior, yet the classical PAC-Bayes KL divergence does
not distinguish uncertainty over predictive behavior from variation among
behaviorally equivalent realizations.

We show that this distinction induces an exact structural decomposition of
classical PAC-Bayes complexity. We formalize behavioral equivalence through a
measurable behavior map and use measure disintegration to decompose probability
measures on the configuration space into a distribution over predictive
behaviors together with conditional distributions over behavioral fibers. This
yields an exact decomposition of the classical PAC-Bayes KL divergence into a
behavior-selection term and a realization-level term given by an expected
conditional KL divergence within behavioral fibers.

We define PAC-Bayes Z-information as the negative of this realization-level
contribution. Consequently, PAC-Bayes Z-information exactly quantifies the gap
between the classical PAC-Bayes KL divergence and the irreducible complexity
associated with uncertainty over predictive behavior. We further show that the
behavior-selection term admits an exact variational characterization: it is the
minimum classical PAC-Bayes KL divergence among all posteriors that induce the
same distribution over predictive behaviors. Equivalently, every posterior
admits a canonical fiber-symmetrized representative with identical predictive
behavior and minimum classical PAC-Bayes complexity.

Finally, we show that symmetry, behavior-preserving directions, fiber geometry,
and invariance under fiber-preserving perturbations arise naturally from the
same behavior-map structure. Together, these results identify predictive
behavior as the natural object of PAC-Bayes complexity, provide a unified
measure-theoretic and geometric characterization of realization multiplicity,
and reveal an exact structural decomposition that is implicit in classical
PAC-Bayes theory.
\end{abstract}

\section{Introduction}

\subsection{Background and Motivation}

PAC-Bayes theory provides some of the strongest generalization guarantees in
statistical learning by relating predictive performance to divergences between
posterior and prior distributions over hypotheses
\citep{mcallester1999pac,seeger2002pac,catoni2007pac}. These results build on
fundamental information-theoretic quantities such as entropy, relative entropy,
cross-entropy, and mutual information, which play a central role in learning,
compression, and statistical inference
\citep{shannon1948,shannon1951prediction,cover2006elements}. Classical
principles such as Minimum Description Length (MDL) similarly interpret learning
through compression and parsimonious representation
\citep{rissanen1978model,grunwald2007minimum}.

A common feature of these approaches is that uncertainty and complexity are
quantified with respect to a chosen representation, whether over hypotheses,
functions, parameters, or observable outcomes. However, predictive risk depends
only on the behavior induced by a hypothesis, not on the particular internal
realization that implements that behavior. This raises a fundamental question:
to what extent does classical PAC-Bayes complexity reflect uncertainty over
predictive behavior, as opposed to variation among equivalent realizations that
have no effect on prediction?

This distinction becomes increasingly important in modern over-parameterized
learning systems. Neural networks often admit many distinct parameter
configurations that induce identical or nearly identical input--output behavior
\citep{neyshabur2017exploring,dinh2017sharp}. Such redundancy has been
empirically associated with improved generalization and robustness, often
through flat minima or low-curvature regions of parameter space
\citep{hochreiter1997flat,keskar2017sharp,chaudhari2019entropy}. At the same
time, sharpness- and curvature-based characterizations are sensitive to
reparameterization and optimization details, making it difficult to interpret
them as intrinsic properties of predictive behavior
\citep{dinh2017sharp,jiang2020fantastic}.

These observations suggest distinguishing uncertainty over predictive behavior
from variation among realizations of the same behavior. If many configurations
induce an identical predictive behavior---a phenomenon we refer to as
\emph{realization multiplicity}---then predictive risk cannot distinguish among
them, whereas configuration-space complexity may still assign distinct
contributions to behaviorally equivalent realizations.

PAC-Bayes theory already permits formulations on arbitrary measurable
hypothesis spaces, including function-space representations
\citep{seeger2002pac}. Likewise, recent work has shown that invariance and
symmetry can reduce effective PAC-Bayes complexity by identifying predictors
that are equivalent under suitable transformations
\citep{lyle2020benefits,behboodi2022pac}. These developments motivate a broader
question. Given an arbitrary measurable notion of behavioral equivalence, does
the classical PAC-Bayes KL divergence admit an exact decomposition into
contributions arising from predictive behavior and contributions arising from
realization multiplicity? If so, what structural properties does such a
decomposition reveal?

\subsection{Behavioral Sufficiency and Realization Multiplicity}

The central observation underlying this work is that empirical and population
risks are invariant to the particular realization of a predictor and depend
only on its induced input--output behavior.

This motivates an equivalence relation on hypotheses:
\[
\theta \sim \theta'
\quad\Longleftrightarrow\quad
\theta \text{ and } \theta'
\text{ induce identical predictive behavior}.
\]

Behavioral equivalence partitions the configuration space into equivalence
classes of realizations that implement the same predictor. Each equivalence
class therefore represents a single predictive behavior together with all of
its realizations. This separation between behavior and realization provides a
natural framework for distinguishing uncertainty over predictive behavior from
variation among equivalent realizations.

For the measure-theoretic development that follows, we realize this structure
through a measurable behavior map
\[
\beta:\Theta\rightarrow\mathcal K,
\]
where $\mathcal K$ is a standard Borel space of predictive behaviors and
$\beta(\theta)$ denotes the predictive behavior induced by configuration
$\theta$. Behavioral equivalence is then precisely equality under the behavior
map:
\[
\theta\sim\theta'
\quad\Longleftrightarrow\quad
\beta(\theta)=\beta(\theta').
\]

For a behavior $k\in\mathcal K$, the set $\beta^{-1}(k)$ contains all
configurations that induce that behavior. In the language of measurable maps,
these sets are the fibers of $\beta$. Throughout the paper, we use the terms
\emph{behavioral equivalence class} and \emph{fiber} interchangeably.

Under standard measure-theoretic assumptions, probability measures on the
configuration space admit a disintegration with respect to the behavior map
\citep{chang1997conditioning,kallenberg2002foundations}. This yields a
decomposition into

\begin{enumerate}
\item a distribution over predictive behaviors; and
\item conditional distributions over realizations within each behavioral fiber.
\end{enumerate}

From this perspective, the classical PAC-Bayes KL divergence naturally
separates into two conceptually distinct sources of divergence:

\begin{enumerate}
\item uncertainty over predictive behavior; and
\item uncertainty over realizations conditional on behavior.
\end{enumerate}

Only the first directly influences predictive risk. The second reflects
variation among realizations that induce the same behavior and therefore
contributes to the classical PAC-Bayes KL divergence without changing
predictive behavior.

Importantly, this decomposition is induced entirely by the behavior map and is
therefore structural rather than parameterization-dependent. It arises from
behavioral equivalence itself rather than from any particular coordinate
system, optimization procedure, or local geometric approximation.

\subsection{Contributions}

Our objective is not to formulate PAC-Bayes analysis on behavior spaces per se,
but to identify the latent structure that an arbitrary measurable notion of
behavioral equivalence induces within classical PAC-Bayes complexity. We show
that the classical PAC-Bayes KL divergence admits an exact decomposition into
behavior-level and realization-level contributions and characterize the
information-theoretic and geometric consequences of this decomposition.

Our main contributions are as follows.

\begin{itemize}
\vspace*{-3pt}

\item \textbf{Behavioral decomposition of PAC-Bayes complexity.}
We formalize behavioral equivalence through an arbitrary measurable behavior
map and develop the associated fiber-based measure-theoretic framework. Within
this framework, we prove that the classical PAC-Bayes KL divergence admits an
exact decomposition into a behavior-selection component and a realization-level
component associated with behavioral fibers.

\vspace*{-3pt}

\item \textbf{PAC-Bayes Z-information.}
We define PAC-Bayes Z-information as the negative expected conditional KL
divergence within behavioral fibers and show that it exactly characterizes the
realization-level contribution to classical PAC-Bayes complexity, thereby
isolating the effect of realization multiplicity.

\vspace*{-3pt}

\item \textbf{Variational characterization of behavior selection complexity.}
We show that the behavior-selection term is the minimum classical PAC-Bayes KL
divergence among all posteriors inducing the same distribution over predictive
behaviors. This identifies behavior selection complexity as the irreducible
contribution associated with predictive behavior itself.

\vspace*{-3pt}

\item \textbf{Unified geometric interpretation.}
Under suitable regularity assumptions, behavioral fibers naturally give rise to
behavior-preserving directions, symmetry, realization multiplicity, and
invariance under fiber-preserving perturbations. These provide complementary
geometric interpretations of the realization-level component identified by the
information-theoretic decomposition.

\end{itemize}
Taken together, these results show that classical PAC-Bayes complexity contains
a latent decomposition into behavior-level and realization-level components.
The remainder of the paper develops this decomposition, characterizes its
information-theoretic and geometric consequences, and shows how it clarifies
the role of realization multiplicity in PAC-Bayes analysis.

\subsection{Organization}

The remainder of the paper is organized as follows.
Section~\ref{sec:setup} introduces behavioral equivalence, behavior maps,
and the measure-theoretic framework used throughout the paper.
Section~\ref{sec:zinfo} develops the fiberwise decomposition induced by
behavioral equivalence and introduces realization entropy and PAC-Bayes
Z-information as measures of realization multiplicity.
Section~\ref{sec:pacbayes} establishes the exact decomposition of the
classical PAC-Bayes KL divergence into behavior-selection and realization-level
terms, develops the corresponding PAC-Bayes analysis, and derives a
variational characterization of the behavior-selection term.
Section~\ref{sec:geometry} develops geometric interpretations in continuous
configuration spaces, relating realization multiplicity to behavioral fibers,
symmetry, behavior-preserving directions, and volume structure.
Section~\ref{sec:stability} studies behavioral invariance under
fiber-preserving perturbations and its relation to classical algorithmic
stability.
Section~\ref{sec:related} discusses related work.
Section~\ref{sec:conclusion} concludes with a summary, implications, and
directions for further research.

\section{Setup and Behavioral Equivalence}
\label{sec:setup}

We formalize the learning setting and introduce behavioral equivalence, which
provides the structural basis for separating uncertainty over predictive
behavior from uncertainty over internal realization.

The key observation is simple. In modern over-parameterized models, many
distinct parameter configurations can induce identical input--output behavior.
Hidden-unit permutations, scaling symmetries, and other sources of non-identifiability
can produce different parameter vectors that implement the same predictor.
Since prediction depends only on behavior and not on the particular realization
that implements it, generalization analysis should distinguish uncertainty over
predictive behavior from redundancy among equivalent realizations. The goal of
this section is to make that distinction precise.

\subsection{Learning Framework}

Let $\mathcal X$ and $\mathcal Y$ denote the input and output spaces, and let
$\mathcal D$ be an unknown distribution over
$\mathcal X\times\mathcal Y$.
We work in the standard statistical learning setting
\citep{vapnik1998statistical,shalev2014understanding}.

The configuration space $\Theta$ indexes hypotheses (internal realizations).
For example, $\Theta$ may be a neural-network parameter space.
Each configuration $\theta\in\Theta$ induces a conditional predictive
distribution
$
P_\theta(\cdot\mid x).
$

Performance is measured using a bounded loss
$
\ell:\Delta(\mathcal Y)\times\mathcal Y\rightarrow[0,1],
$
where $\Delta(\mathcal Y)$ denotes the set of probability distributions on
$\mathcal Y$.

The population and empirical risks are
\[
L(\theta)
=
\mathbb E_{(x,y)\sim\mathcal D}
\bigl[
\ell(P_\theta(\cdot\mid x),y)
\bigr],
\qquad
\hat L_S(\theta)
=
\frac1n
\sum_{i=1}^{n}
\ell(P_\theta(\cdot\mid x_i),y_i),
\]
where
$
S=\{(x_i,y_i)\}_{i=1}^{n}.
$

As in PAC-Bayes theory, we consider stochastic predictors
$
Q\in\mathcal P(\Theta).
$

Their population and empirical risks are
\[
L(Q)
=
\mathbb E_{\theta\sim Q}[L(\theta)],
\qquad
\hat L_S(Q)
=
\mathbb E_{\theta\sim Q}[\hat L_S(\theta)],
\]
following the standard PAC-Bayes formulation
\citep{mcallester1999pac,seeger2002pac,catoni2007pac}.

\subsection{Behavior Maps and Behavioral Equivalence}

Many learning systems admit multiple internal realizations that implement the
same predictive behavior. To formalize this idea, we introduce a measurable
behavior map
$
\beta:\Theta\rightarrow\mathcal K,
$
where $\mathcal K$ is a standard Borel space of predictive behaviors (for
example, conditional distributions
$x\mapsto P_\theta(\cdot\mid x)$), and
$\beta(\theta)$ denotes the predictive behavior induced by configuration
$\theta$.

\begin{definition}[Behavioral Equivalence]
Two configurations
$
\theta,\theta'\in\Theta
$
are behaviorally equivalent, written
$
\theta\sim\theta',
$
if
$
\beta(\theta)
=
\beta(\theta').
$
\end{definition}

Equivalently, behavioral equivalence identifies configurations that induce the
same predictive behavior as represented by the behavior map $\beta$. In the
examples considered throughout this paper, $\beta$ records the conditional
predictive distribution induced by a configuration.

For simplicity, we state behavioral equivalence in terms of exact equality of
predictive behavior. All results continue to hold if behavioral equivalence is
defined modulo equality almost everywhere with respect to the input
distribution.

For a behavior
$
k\in\mathcal K,
$
the set
$
F_k
=
\beta^{-1}(k)
=
\{\theta\in\Theta:\beta(\theta)=k\}
$
contains all configurations that realize behavior $k$.
Following standard terminology, we refer to $F_k$ as the
\emph{fiber} associated with behavior $k$
\citep{kallenberg2002foundations,parthasarathy1967probability}.

Thus, each behavior corresponds to a fiber of internally distinct but
predictively equivalent realizations.

For example, permuting hidden units in a neural network while applying the
corresponding permutation to outgoing weights leaves the predictive behavior
unchanged (See Appendix \ref{app:worked:permutations}). The original and permuted parameter
vectors therefore belong to the same fiber.

More generally, fibers may arise from parameter symmetries, redundant hidden
units, scaling invariances, or other sources of non-identifiability. Such
redundancy is a well-known feature of over-parameterized models
\citep{neyshabur2017exploring,dinh2017sharp}.

Although the theory developed in this paper is formulated in terms of the
behavior map $\beta$, it is often helpful to think of $\mathcal K$ as
playing the role of a quotient representation of the configuration space,
where behaviorally equivalent realizations are treated as equivalent.

\subsection{Induced Distribution over Behaviors}

Any distribution
$
Q\in\mathcal P(\Theta)
$
induces a distribution over predictive behaviors through the pushforward
measure
$
\pi_Q
=
\beta_\#Q
=
Q\circ\beta^{-1}
\in
\mathcal P(\mathcal K).
$

Intuitively, $\pi_Q$ records how much probability mass $Q$ assigns to each
predictive behavior while ignoring how that mass is distributed among
realizations within a fiber.

\paragraph{\textbf{Behavioral Sufficiency.}}

The induced distribution $\pi_Q$ captures all information in $Q$ that can
affect population or empirical risk. This is not an additional modeling
assumption. Rather, it follows directly from behavioral equivalence:
configurations within the same fiber induce the same predictive behavior and
therefore incur the same loss.

\subsection{Loss Invariance and Behavioral Sufficiency}

\begin{lemma}[Loss Invariance under Behavioral Equivalence]
\label{lem:loss-invariance}

If
$
\theta\sim\theta',
$
then
\[
L(\theta)
=
L(\theta'),
\qquad
\hat L_S(\theta)
=
\hat L_S(\theta').
\]
\end{lemma}
\begin{proof}

Behavioral equivalence implies that $\theta$ and $\theta'$ induce the same
predictive behavior. Hence
\[
\ell(P_\theta(\cdot\mid x),y)
=
\ell(P_{\theta'}(\cdot\mid x),y)
\]
for every example $(x,y)$. Taking expectations with respect to
$\mathcal D$ yields
$
L(\theta)=L(\theta'),
$
while averaging over the sample $S$ yields
$
\hat L_S(\theta)=\hat L_S(\theta').
$
\end{proof}

The lemma implies that both population and empirical risk are constant on each
fiber. Consequently, there exist functions
$
L_{\mathcal K},
\hat L_{S,\mathcal K},
$
defined on the behavior space $\mathcal K$ such that
\[
L(\theta)
=
L_{\mathcal K}(\beta(\theta)),
\qquad
\hat L_S(\theta)
=
\hat L_{S,\mathcal K}(\beta(\theta)).
\]
\begin{proposition}[Behavioral Sufficiency]
\label{prop:sufficiency}

If
$
\pi_Q
=
\pi_{Q'},
$
then
\[
L(Q)
=
L(Q'),
\qquad
\hat L_S(Q)
=
\hat L_S(Q').
\]
\end{proposition}

\begin{proof}

Using the factorization above,

\[
L(Q)
=
\mathbb E_{\theta\sim Q}
\Big[
L_{\mathcal K}(\beta(\theta))
\Big]
=
\mathbb E_{k\sim\pi_Q}
\bigl[
L_{\mathcal K}(k)
\bigr].
\]
The same representation holds for $Q'$.
If $\pi_Q=\pi_{Q'}$, the expectations coincide.
The argument for empirical risk is identical.
\end{proof}

Proposition~\ref{prop:sufficiency} shows that both population and empirical
risk factors are affected by the behavior map. Consequently, there exist functionals,
which we denote by the same symbols for convenience, such that
\[
L(Q)
=
L(\pi_Q),
\qquad
\hat L_S(Q)
=
\hat L_S(\pi_Q).
\]
Thus, all quantities relevant to prediction depend on a posterior only through
its induced distribution over predictive behaviors.

In particular, any two posteriors with the same induced behavior distribution
are indistinguishable from the perspective of empirical and population risk.

\begin{remark}
Behavioral equivalence is not introduced as an approximation or modeling
assumption. Rather, it identifies configurations that are indistinguishable
from the perspective of prediction. Since empirical and population risks depend
only on predictive behavior, behavioral equivalence isolates the aspects of a
hypothesis that are relevant to generalization from those that merely reflect
its internal realization.
\end{remark}

\begin{remark}
Proposition~\ref{prop:sufficiency} shows that predictive behavior, rather than
parameterization, is the natural object of analysis for generalization.
Different posteriors over configurations may induce identical distributions
over behaviors and therefore identical risks, despite assigning probability mass
to different realizations.
\end{remark}

\subsection{Measure-Theoretic Assumptions}

To separate uncertainty across behaviors from uncertainty within a behavior, we
will later decompose probability measures along fibers.

We assume that
$
(\Theta,\mathcal F_\Theta)
$
and
$
(\mathcal K,\mathcal F_{\mathcal K})
$
are standard Borel spaces and that the behavior map
$
\beta:\Theta\rightarrow\mathcal K
$
is measurable.\footnote{Standard Borel spaces are measurable spaces arising
from the Borel $\sigma$-algebra of a Polish (complete separable metric) space.
They include essentially all parameter and function spaces commonly used in
statistical learning, including Euclidean spaces and many spaces of
probability measures equipped with their natural Borel structures
\citep{kallenberg2002foundations,parthasarathy1967probability}.}

Under these assumptions, standard disintegration theorems guarantee the
existence of a family of conditional probability measures
$
Q(\cdot\mid k)
$
such that, for every measurable set
$A\in\mathcal F_\Theta$,
\[
Q(A)
=
\int_{\mathcal K}
Q(A\mid k)\,
\pi_Q(dk).
\]
Equivalently,
\[
Q(d\theta)
=
Q(d\theta\mid k)\,
\pi_Q(dk).
\]
The conditional measure
$
Q(d\theta\mid k)
$
describes uncertainty among realizations that induce the same predictive
behavior $k$
\citep{kallenberg2002foundations,parthasarathy1967probability,
bogachev2007measure}.

Readers unfamiliar with fibers or measure disintegration may consult
Appendix~\ref{app:preliminaries}, which provides a brief review of the
required background.

\subsection{Role in Generalization}

Classical generalization theory controls deviations between empirical and
population risk using complexity measures such as VC dimension, stability,
Rademacher complexity, and PAC-Bayes KL divergence
\citep{vapnik1998statistical,bousquet2002stability,mcallester1999pac}.

In the present setting, complexity measures defined on the configuration space
$\Theta$ combine two conceptually distinct sources of uncertainty:

\begin{enumerate}
\item uncertainty over predictive behaviors;
\item uncertainty among realizations that implement a fixed behavior.
\end{enumerate}

By Proposition~\ref{prop:sufficiency}, only the first affects population or
empirical risk. The second reflects \emph{realization multiplicity} arising
from over-parameterization.

This distinction between behavioral uncertainty and realization multiplicity is
the central organizing principle of the paper. The next section shows that the
PAC-Bayes KL divergence admits an exact fiberwise decomposition into a
behavior-selection term and a realization-level term. PAC-Bayes Z-information is
defined as the negative of the realization-level contribution and therefore
quantifies the portion of PAC-Bayes complexity associated with the
redistribution of posterior mass among behaviorally equivalent realizations.

Moreover, passing from a posterior on configurations to its induced
distribution on predictive behaviors can be viewed as an application of the
data-processing principle for relative entropy. The next section develops this
connection and shows that the behavior selection complexity term admits a
variational characterization within the classical PAC-Bayes framework: among
all posteriors inducing the same distribution over predictive behaviors, it is
the minimum achievable PAC-Bayes complexity. Equivalently, each behavior-level
posterior admits a canonical fiber-symmetrized representative whose classical
PAC-Bayes complexity coincides with the behavior-selection term. This
characterization will play a central role in the development that follows.

\section{Behavior-Realization Decomposition and PAC-Bayes Z-Information}
\label{sec:zinfo}

Section~\ref{sec:setup} established that predictive risk depends only on the
distribution of predictive behaviors induced by a posterior. This section
develops the central information-theoretic decomposition underlying the
remainder of the paper.

The key result is an exact decomposition of the classical PAC-Bayes
complexity term into a behavior-level component and a realization-level
component. PAC-Bayes Z-information emerges as the negative of the latter and
quantifies the portion of PAC-Bayes complexity attributable to redistributing
probability mass among behaviorally equivalent realizations.

Throughout, the behavior map
$\beta:\Theta\rightarrow\mathcal K$
induces behavior-level distributions
\[
\pi_P=\beta_\#P,
\qquad
\pi_Q=\beta_\#Q,
\]
where $\beta_\#P:=P\circ\beta^{-1}$ and
$\beta_\#Q:=Q\circ\beta^{-1}$ denote the pushforward measures of $P$ and
$Q$ through the behavior map. Equivalently, for every measurable set
$A\subseteq\mathcal K$,
\[
\pi_P(A)=P(\beta^{-1}(A)),
\qquad
\pi_Q(A)=Q(\beta^{-1}(A)).
\]
Because the quantities developed in this paper are relative rather than
absolute, our analysis is based on the relative entropy between posterior and
prior distributions rather than on the entropy defined within individual fibers.
This avoids introducing fiber-specific reference measures and yields
quantities that arise naturally in PAC-Bayes theory.

\subsection{Behavior--Realization Decomposition of Relative Entropy}

The central object in PAC-Bayes theory is the relative entropy
$\mathrm{KL}(Q\|P)$ between a posterior $Q$ and a prior $P$. Because this
quantity is computed on the configuration space $\Theta$, it combines variation
across predictive behaviors with variation among realizations of a fixed
behavior.

The decomposition extends to the extended-real setting when
$\mathrm{KL}(Q\|P)=\infty$. We focus on the finite-KL regime because it is the
one relevant to classical PAC-Bayes bounds.

The following result separates these contributions exactly.

\begin{proposition}[Behavior--Realization KL Decomposition]
\label{prop:kl-decomposition}

Assume $\mathrm{KL}(Q\|P)<\infty$ and let
\[
Q(d\theta)=Q(d\theta\mid k)\,\pi_Q(dk),
\qquad
P(d\theta)=P(d\theta\mid k)\,\pi_P(dk)
\]
denote disintegrations with respect to the behavior map
$\beta:\Theta\rightarrow\mathcal K$.
Then $\pi_Q\ll\pi_P$ and

\[
\mathrm{KL}(Q\|P)
=
\mathrm{KL}(\pi_Q\|\pi_P)
+
\mathbb E_{k\sim\pi_Q}
\!\left[
\mathrm{KL}
\!\left(
Q(\cdot\mid k)
\,\middle\|\,
P(\cdot\mid k)
\right)
\right],
\]
where $``\ll"$ denotes absolute continuity of measures.
\end{proposition}

The decomposition is a direct consequence of the chain rule for relative
entropy under measure disintegration
\citep[see, e.g.,][]{kallenberg2002foundations,bogachev2007measure}.
Its importance here is interpretive rather than technical: it separates
behavior selection complexity from realization-level complexity.

\begin{proof}[Proof sketch]

Because $\mathrm{KL}(Q\|P)<\infty$, we have $Q\ll P$, and hence
$\pi_Q\ll\pi_P$.

Using the disintegrations of $Q$ and $P$ with respect to the behavior map
$\beta$, the chain rule for relative entropy yields a decomposition into
a behavior-selection term and a conditional within-fiber term. Rearranging gives the stated identity.

A complete proof is given in
Appendix~\ref{app:kl-decomposition}.
\end{proof}

The first term measures the divergence between distributions over predictive
behaviors. The second measures the divergence between posterior and prior
conditional distributions within behavioral fibers.

Thus, the decomposition separates the divergence associated with selecting
predictive behaviors from the divergence associated with redistributing mass among
realizations of those behaviors.

\begin{corollary}[Data Processing Inequality]
\label{cor:dpi}
\[
\mathrm{KL}(\pi_Q\|\pi_P)
\le
\mathrm{KL}(Q\|P).
\]
\end{corollary}

\begin{proof}

The conditional KL divergence in
Proposition~\ref{prop:kl-decomposition}
is nonnegative.

\end{proof}

Corollary~\ref{cor:dpi} is precisely the data-processing inequality for
relative entropy applied to the measurable map
$\beta:\Theta\rightarrow\mathcal K$.

In particular, passing from a posterior on configurations to its induced
distribution on predictive behaviors can only decrease relative entropy.
The lost information is exactly the realization-level term identified in
Proposition~\ref{prop:kl-decomposition}.

\begin{remark}

The decomposition in Proposition~\ref{prop:kl-decomposition}
strictly refines the data-processing inequality.
The data-processing inequality asserts only that
$\mathrm{KL}(\pi_Q\|\pi_P)\le\mathrm{KL}(Q\|P)$,
whereas the decomposition identifies the exact nonnegative remainder,
namely the expected conditional divergence within fibers.
PAC-Bayes Z-information is defined as the negative of this remainder.

\end{remark}

\subsection{PAC-Bayes Z-Information}

The decomposition above suggests isolating the realization-level contribution.

\begin{definition}[PAC-Bayes Z-Information]
\label{def:zinfo}

The \emph{PAC-Bayes Z-information} of a posterior $Q$ relative to a prior $P$
is
\[
\mathcal Z_{\mathrm{PB}}(Q\|P)
=
-
\mathbb E_{k\sim\pi_Q}
\!\left[
\mathrm{KL}
\!\left(
Q(\cdot\mid k)
\,\middle\|\,
P(\cdot\mid k)
\right)
\right].
\]
\end{definition}

By construction,
$\mathcal Z_{\mathrm{PB}}(Q\|P)\le0$,
with equality if and only if
$Q(\cdot\mid k)=P(\cdot\mid k)$
for $\pi_Q$-almost every $k$.

Combining Definition~\ref{def:zinfo} with
Proposition~\ref{prop:kl-decomposition}
immediately yields the following identity.

\begin{proposition}[Z-Information Identity]
\label{prop:zinfo-identity}
\[
\mathcal Z_{\mathrm{PB}}(Q\|P)
=
\mathrm{KL}(\pi_Q\|\pi_P)
-
\mathrm{KL}(Q\|P).
\]
\end{proposition}

\begin{proof}

Rearrange the identity in
Proposition~\ref{prop:kl-decomposition}
and substitute Definition~\ref{def:zinfo}.

\end{proof}

PAC-Bayes Z-information is therefore not an additional divergence.
Rather, it is an alternative representation of the gap between
configuration-space complexity and behavior-space complexity.

A value of
$\mathcal Z_{\mathrm{PB}}(Q\|P)=0$
indicates that the posterior and prior induce identical distributions
within every fiber, so that all complexity is attributable to
uncertainty over predictive behavior.

Negative values arise precisely when the posterior differs from the prior
within behavioral fibers.
In this case, the classical PAC-Bayes complexity exceeds the
behavior selection complexity by
$-\mathcal Z_{\mathrm{PB}}(Q\|P)$.

\subsection{Variational Characterization}

The decomposition admits a useful variational interpretation that clarifies the
relationship between behavior selection complexity and classical PAC-Bayes theory.

Given a posterior $Q$, define the \emph{fiber-symmetrized posterior}
\[
Q^\star(d\theta)
=
\int_{\mathcal K}
P(d\theta\mid k)\,
\pi_Q(dk).
\]
The distribution $Q^\star$ preserves the behavior-level distribution
$\pi_Q$ while replacing each fiberwise conditional distribution
$Q(\cdot\mid k)$ with the corresponding prior conditional
$P(\cdot\mid k)$.
Because $Q$ and $Q^\star$ induce the same distribution over predictive
behaviors, Proposition~\ref{prop:sufficiency} implies that
\[L(Q^\star)=L(Q)\; \mathrm{and}\;
\hat L_S(Q^\star)=\hat L_S(Q)\].
\begin{theorem}[Fiber-Symmetrized Optimal Representative]
\label{thm:optimal-representative}

The fiber-symmetrized posterior satisfies
$\pi_{Q^\star}=\pi_Q$ and
\[
\mathrm{KL}(Q^\star\|P)
=
\mathrm{KL}(\pi_Q\|\pi_P).
\]
Moreover,
\[
\mathrm{KL}(\pi_Q\|\pi_P)
=
\inf_{Q':\,\pi_{Q'}=\pi_Q}
\mathrm{KL}(Q'\|P).
\]
\end{theorem}

\paragraph{Interpretation.}

Theorem~\ref{thm:optimal-representative} identifies the behavior-selection
divergence $\mathrm{KL}(\pi_Q\|\pi_P)$ as the canonical behavior-level
component of classical PAC-Bayes complexity. Among all posteriors inducing the
same distribution over predictive behaviors, the fiber-symmetrized posterior
achieves the minimum classical PAC-Bayes KL divergence. Consequently, any
excess classical PAC-Bayes complexity beyond
$\mathrm{KL}(\pi_Q\|\pi_P)$ arises entirely from how posterior probability is
distributed among behaviorally equivalent realizations. PAC-Bayes
Z-information quantifies this realization-level contribution exactly, making
explicit a latent decomposition of the classical PAC-Bayes KL divergence into
behavior-selection and realization-level components.

\begin{proof}
Since $Q^\star$ and $Q$ share the same behavior-level marginal,
$\pi_{Q^\star}=\pi_Q$.

Applying Proposition~\ref{prop:kl-decomposition} to $Q^\star$ yields
$\mathrm{KL}(Q^\star\|P)=\mathrm{KL}(\pi_Q\|\pi_P)$
because the conditional divergence term vanishes identically.

For any posterior $Q'$ satisfying $\pi_{Q'}=\pi_Q$,
\[
\mathrm{KL}(Q'\|P)
=
\mathrm{KL}(\pi_Q\|\pi_P)
+
\mathbb E_{k\sim\pi_Q}
\!\left[
\mathrm{KL}
\!\left(
Q'(\cdot\mid k)
\,\middle\|\,
P(\cdot\mid k)
\right)
\right].
\]
The conditional KL term is nonnegative.

\end{proof}

\begin{remark}
Theorem~\ref{thm:optimal-representative} shows that every distribution over
predictive behaviors admits a canonical configuration-space representative
whose classical PAC-Bayes complexity equals the corresponding
behavior-selection divergence. Consequently, the present framework should be
viewed not as introducing a new PAC-Bayes bound, but as revealing an exact
variational characterization and latent structural decomposition that is
already implicit within classical PAC-Bayes theory.
\end{remark}

\subsection{Illustration: Hidden-Unit Permutation Symmetry}

The behavior--realization decomposition is particularly transparent in models
with parameter symmetries. Consider a neural network with $h$ hidden units and
a prior that is invariant under permutations of those units. Permuting hidden
units together with their incident weights leaves the induced input--output
behavior unchanged. Consequently, each predictive behavior corresponds to a
behavioral fiber containing all parameter configurations related by such
permutations.

If the posterior concentrates near one representative realization while the
prior remains approximately uniform over the corresponding behavioral fiber,
then the induced predictive behavior remains unchanged, but the posterior and
prior differ substantially within the fiber. Under this idealized setting, the
realization-level contribution satisfies
\[
-\mathcal Z_{\mathrm{PB}}(Q\|P)
\approx
\log(h!),
\]
reflecting the combinatorial multiplicity of behaviorally equivalent
realizations.

Thus, realization multiplicity alone can contribute substantially to the
classical PAC-Bayes KL divergence even though every realization within the
fiber induces exactly the same predictive behavior. Importantly, this
contribution reflects uncertainty over internal realizations rather than
uncertainty over predictive behavior itself.

Conversely, if both the prior and posterior distribute probability identically
within each behavioral fiber---for example, under an exact
symmetry-preserving Bayesian posterior---then
$\mathcal Z_{\mathrm{PB}}(Q\|P)=0$, and the classical PAC-Bayes KL divergence
already coincides with the behavior-selection divergence.

Although hidden-unit permutation symmetry provides a familiar illustration,
the behavior--realization decomposition developed in this paper applies to
arbitrary measurable notions of behavioral equivalence and is not restricted
to finite symmetry groups. Appendix~\ref{app:worked} presents explicit
computations illustrating the decomposition theorem, the variational
characterization, and the conditional KL calculation underlying the
approximation above.

The decomposition developed in this section provides the foundation for the
PAC-Bayes analysis that follows. The next section derives behavior-level
generalization bounds, establishes their relationship to classical PAC-Bayes
theory, and characterizes the role of realization multiplicity in
generalization.

\section{Behavior-Space Reformulation of PAC-Bayes}
\label{sec:pacbayes}

Section~\ref{sec:zinfo} established that behavioral equivalence induces an
exact decomposition of the classical PAC-Bayes KL divergence into a
behavior-selection component and a realization-level component. We now combine
this decomposition with the observation that empirical and population risks
depend only on the induced distribution over predictive behaviors.

Because empirical and population risks factor through the behavior map, the
classical PAC-Bayes theorem admits an equivalent formulation on the measurable
space of predictive behaviors $\mathcal K$. The resulting complexity term is
precisely $\mathrm{KL}(\pi_Q\|\pi_P)$, the divergence between the
behavior-level distributions induced by the posterior and prior. By
Theorem~\ref{thm:optimal-representative}, this quantity is the minimum
configuration-space PAC-Bayes complexity among all posteriors inducing the same
distribution over predictive behaviors. Consequently, the behavior-space
formulation identifies the irreducible contribution of predictive behavior to
classical PAC-Bayes complexity, while PAC-Bayes Z-information quantifies the
realization-level contribution arising from the distribution of posterior mass
among behaviorally equivalent realizations.

\subsection{Classical PAC-Bayes Bounds}

For completeness, we recall a standard PAC-Bayes inequality
\citep{mcallester1999pac,seeger2002pac,catoni2007pac}.

\begin{theorem}[Classical PAC-Bayes Bound]
\label{thm:classical-pacbayes}

Let $P$ be a prior over $\Theta$ chosen independently of
$S\sim\mathcal D^n$. Then, with probability at least $1-\delta$ over the draw
of $S$, simultaneously for all posteriors $Q$,
\[
\mathrm{KL}
\!\left(
\hat L_S(Q)
\,\middle\|\,
L(Q)
\right)
\le
\frac1n
\left(
\mathrm{KL}(Q\|P)
+
\log\frac{2\sqrt n}{\delta}
\right).
\]
\end{theorem}

The complexity term $\mathrm{KL}(Q\|P)$ is computed on the configuration
space $\Theta$, and therefore combines variation across predictive behaviors
with variation among realizations that induce the same behavior.

\subsection{PAC-Bayes Analysis on the Behavior Space}

Let $\pi_P$ and $\pi_Q$ denote the pushforwards of the prior and posterior
under the behavior map $\beta$.

Because empirical and population risk depend only on the induced behavior
distribution (Proposition~\ref{prop:sufficiency}), PAC-Bayes analysis may be
carried out directly on the behavior space.

\begin{theorem}[Behavior-Space Formulation of the PAC-Bayes Bound]
\label{thm:behavior-pacbayes}

Let $P$ be a prior over $\Theta$ chosen independently of
$S\sim\mathcal D^n$. Then, with probability at least $1-\delta$ over the draw
of $S$, simultaneously for all posteriors $Q$,

\[
\mathrm{KL}
\!\left(
\hat L_S(Q)
\,\middle\|\,
L(Q)
\right)
\le
\frac1n
\left(
\mathrm{KL}(\pi_Q\|\pi_P)
+
\log\frac{2\sqrt n}{\delta}
\right).
\]
\end{theorem}

\begin{proof}

Because $\mathcal K$ is a standard Borel space and
$\pi_P,\pi_Q\in\mathcal P(\mathcal K)$, the classical PAC-Bayes theorem
applies directly to the measurable hypothesis space $\mathcal K$.

Applying Theorem~\ref{thm:classical-pacbayes} with prior $\pi_P$ and posterior
$\pi_Q$ yields
\[
\mathrm{KL}
\!\left(
\hat L_S(\pi_Q)
\,\middle\|\,
L(\pi_Q)
\right)
\le
\frac1n
\left(
\mathrm{KL}(\pi_Q\|\pi_P)
+
\log\frac{2\sqrt n}{\delta}
\right).
\]
Proposition~\ref{prop:sufficiency} implies
$L(\pi_Q)=L(Q)$ and $\hat L_S(\pi_Q)=\hat L_S(Q)$,
yielding the result.
\end{proof}

Theorem~\ref{thm:behavior-pacbayes} is not a new concentration inequality.
Rather, it is the classical PAC-Bayes theorem applied to the measurable
behavior space induced by the behavior map. Its significance lies not in
introducing a new PAC-Bayes inequality, but in showing that predictive
behavior is the natural level at which to measure PAC-Bayes complexity. The
corresponding complexity term is precisely the behavior-selection component of
the exact decomposition developed in Section~\ref{sec:zinfo}. By
Theorem~\ref{thm:optimal-representative}, this quantity is exactly the minimum
classical PAC-Bayes complexity among all posteriors that induce the same
distribution over predictive behaviors.

Consequently, the complexity term in
Theorem~\ref{thm:behavior-pacbayes}
depends only on uncertainty over predictive behavior and is insensitive to how
posterior mass is distributed among behaviorally equivalent realizations.
\subsection{Complexity Decomposition}

By Proposition~\ref{prop:zinfo-identity},
\[
\mathrm{KL}(\pi_Q\|\pi_P)
=
\mathrm{KL}(Q\|P)
+
\mathcal Z_{\mathrm{PB}}(Q\|P).
\]
Since $\mathcal Z_{\mathrm{PB}}(Q\|P)\le 0$, the behavior selection complexity
never exceeds the classical configuration-space complexity.

\begin{proposition}[Complexity Gap]
\label{prop:complexity-gap}

For any prior $P$ and posterior $Q$,
$\mathrm{KL}(Q\|P)-\mathrm{KL}(\pi_Q\|\pi_P)
=
-\mathcal Z_{\mathrm{PB}}(Q\|P)$.
In particular,
$\mathrm{KL}(\pi_Q\|\pi_P)\le \mathrm{KL}(Q\|P)$.

\end{proposition}

\begin{proof}

Immediate from Proposition~\ref{prop:zinfo-identity}.

\end{proof}

Thus, PAC-Bayes Z-information exactly characterizes the gap between classical
configuration-space complexity and behavior selection complexity. Equivalently, it
measures the portion of classical PAC-Bayes complexity attributable to
variation within behavioral fibers.

\subsection{Behavior-Level Bound Expressed Using Z-Information}

Combining Theorem~\ref{thm:behavior-pacbayes} with
Proposition~\ref{prop:zinfo-identity} yields the following equivalent form.

\begin{corollary}[Behavior-space formulation of the PAC-Bayes bound with Z-Information]
\label{cor:zinfo-pacbayes}

With probability at least $1-\delta$ over the draw of $S$,
simultaneously for all posteriors $Q$,
\[
\mathrm{KL}
\!\left(
\hat L_S(Q)
\,\middle\|\,
L(Q)
\right)
\le
\frac1n
\left(
\mathrm{KL}(Q\|P)
+
\mathcal Z_{\mathrm{PB}}(Q\|P)
+
\log\frac{2\sqrt n}{\delta}
\right).
\]
\end{corollary}

Because $\mathcal Z_{\mathrm{PB}}(Q\|P)\le0$, the complexity term in
Corollary~\ref{cor:zinfo-pacbayes} is never larger than the classical
PAC-Bayes complexity term.

Corollary~\ref{cor:zinfo-pacbayes} does not introduce a new concentration
inequality. Rather, it rewrites the behavior selection complexity term in a form
that makes the realization-level contribution explicit.

\subsection{Interpretation}

The results above separate two distinct contributions to PAC-Bayes
complexity:

\begin{enumerate}
\item uncertainty over predictive behaviors;
\item variation among realizations that implement a fixed behavior.
\end{enumerate}

The first contribution is captured by
$\mathrm{KL}(\pi_Q\|\pi_P)$ and is the only component relevant to predictive
risk. The second is captured by
$-\mathcal Z_{\mathrm{PB}}(Q\|P)$ and reflects the redistribution of posterior
mass within behavioral fibers.

Theorem~\ref{thm:optimal-representative} showed that
$\mathrm{KL}(\pi_Q\|\pi_P)$ is the minimum classical PAC-Bayes complexity
among all posteriors that induce the same distribution over predictive
behaviors. Consequently, the behavior-level bound can be viewed as the
PAC-Bayes bound obtained after removing complexity that is irrelevant to
prediction while preserving the induced behavior distribution.

Consequently, PAC-Bayes Z-information quantifies the extent to which
classical PAC-Bayes complexity arises from realization multiplicity rather
than uncertainty about predictive behavior itself.

In models with substantial symmetry or over-parameterization, this
realization-level contribution can be significant. When posterior and prior
induce identical conditional distributions within each fiber,
$\mathcal Z_{\mathrm{PB}}(Q\|P)=0$, and classical PAC-Bayes complexity
coincides exactly with behavior selection complexity.

\subsection{Summary}

Because predictive risk depends only on the induced distribution over
behaviors, PAC-Bayes analysis can be formulated directly on the behavior space
associated with the map $\beta:\Theta\rightarrow\mathcal K$. The resulting
complexity term, $\mathrm{KL}(\pi_Q\|\pi_P)$, depends only on uncertainty over
predictive behavior and is exactly the minimum classical PAC-Bayes complexity
among all posteriors inducing the same behavior distribution.

PAC-Bayes Z-information quantifies the gap between configuration-space and
behavior selection complexity. Equivalently, it measures the portion of classical
PAC-Bayes complexity attributable to variation within behavioral fibers. The
Behavior-space formulation of the PAC-Bayes bound is therefore never looser than the corresponding
classical PAC-Bayes bound and coincides with it precisely when the posterior and
prior induce identical conditional distributions within each fiber.

\section{Realization Multiplicity, Symmetry, and Flat Minima}
\label{sec:geometry}

Sections~\ref{sec:setup}--\ref{sec:pacbayes} develop the
behavior--realization decomposition in a fully measure-theoretic setting. The
purpose of this section is not to strengthen those results, but to provide
geometric intuition for realization multiplicity in continuous configuration
spaces.

Accordingly, the discussion below should be viewed as an interpretation of the
measure-theoretic framework rather than a prerequisite for it. Geometric
assumptions are introduced only where needed to discuss reference measures,
tangent directions, and related notions.

Throughout this section, we use the terms \emph{fiber} and
\emph{behavioral equivalence class} interchangeably.

\subsection{Fibers as Geometric Objects}

Let $\Theta\subseteq\mathbb R^d$ be a continuous configuration space, and let
$\beta:\Theta\rightarrow\mathcal K$ denote the behavior map introduced in
Section~\ref{sec:setup}.

For a behavior $k\in\mathcal K$, the corresponding fiber is
\[
F_k=\beta^{-1}(k).
\]

In general, fibers need not possess a smooth manifold structure. The geometric
discussion below is intended only to provide intuition for settings in which
fibers admit sufficient regularity for notions such as reference measures and
tangent directions to be defined.

A fiber contains all configurations that induce the same predictive behavior.
Consequently, moving within a fiber changes the realization of a predictor
without changing the predictor itself.

Behavioral equivalence therefore decomposes parameter-space variation into
(i) variation across fibers, which changes predictive behavior, and
(ii) variation within fibers, which changes only the realization of that
behavior. This geometric distinction mirrors the behavior
–realization decomposition developed in
Sections~\ref{sec:zinfo} and~\ref{sec:pacbayes}.

\subsection{Symmetry and Realization Multiplicity}

Many over-parameterized models admit symmetries that preserve predictive
behavior. Neural networks provide a familiar example.

Consider a feedforward network with two hidden units. Simultaneously
permuting the incoming and outgoing weights of those units leaves the computed
function unchanged, even though the parameter vector changes.
The resulting parameter configurations therefore belong to the same fiber.

More generally, permutation symmetries, scaling symmetries, redundant hidden
units, and other forms of non-identifiability can produce multiple
configurations that realize the same predictive behavior
\citep{neyshabur2017exploring,dinh2017sharp}.

These symmetries generate realization multiplicity: many distinct points in
parameter space correspond to a single predictive behavior.

\begin{proposition}[Symmetry-Induced Multiplicity]
\label{prop:symmetry-zinfo}

Let $F_k$ be a finite symmetry orbit of size $m$, and suppose
$P(\cdot\mid k)$ is uniform on that orbit. Then
\[
0
\le
\mathrm{KL}
\!\left(
Q(\cdot\mid k)
\,\middle\|\,
P(\cdot\mid k)
\right)
\le
\log m.
\]
\end{proposition}

\begin{proof}

Nonnegativity follows from the basic properties of KL divergence.
The maximum occurs when $Q(\cdot\mid k)$ concentrates on a single realization,
in which case the divergence from the uniform distribution equals $\log m$.

\end{proof}

The proposition shows that increasing the size of a symmetry orbit increases
the maximum possible realization-level contribution to the classical
PAC-Bayes KL divergence while leaving predictive behavior unchanged.
Consequently, realization multiplicity contributes directly to the within-fiber
term identified by PAC-Bayes Z-information (see
Appendix~\ref{app:proof-symmetry-zinfo} for the complete proof).

\subsection{Reference Measures and Continuous Realization Multiplicity}

In continuous models, realization multiplicity is often expressed not through
finite symmetry orbits but through extended subsets of parameter space that
implement the same predictive behavior.

To make this idea precise, let $\mu_k$ denote a reference measure on a fiber
$F_k$. When fibers admit additional geometric structure, $\mu_k$ may be chosen
to reflect that structure (for example, as a Hausdorff measure on a smooth
submanifold representation of the fiber). The quantity
\[
\mu_k(F_k)
\]
measures the size of the realization set associated with behavior $k$
relative to the chosen reference measure.

The numerical value depends on the chosen reference measure, but the
underlying interpretation is invariant: some predictive behaviors admit many
realizations, whereas others admit comparatively few. Accordingly,
realization multiplicity in continuous spaces should be understood relative to
a specified reference measure rather than as an intrinsic geometric volume.

A fiber with a larger reference-measure size corresponds, relative to the
chosen reference measure, to a behavior that can be realized by a larger set
of configurations. Conversely, a fiber with a smaller reference-measure size
corresponds to a behavior supported by a more restricted set of
configurations.

From this perspective, realization multiplicity may be viewed geometrically as
the reference-measure size of the set of configurations associated with a fixed
predictive behavior. This geometric notion is the continuous-space analogue of
the finite symmetry-orbit multiplicity described in
Proposition~\ref{prop:symmetry-zinfo}.

This intuition is reflected formally in the within-fiber KL divergence of
Section~\ref{sec:zinfo}, which compares posterior concentration with the
reference distribution supplied by the conditional prior. In particular,
PAC-Bayes Z-information is defined entirely through conditional relative
entropy and therefore does not require any canonical notion of geometric
volume. The reference measure serves only as geometric intuition for
understanding realization multiplicity in continuous spaces.

\subsection{Fiber Directions and Flatness}
\label{sec:flatness-zinfo}

The geometric consequences of realization multiplicity are closely related to,
but conceptually distinct from, the notion of flat minima
\citep{hochreiter1997flat,keskar2017sharp}.

A key distinction should be emphasized. Behavior-preserving directions are
directions along which predictive behavior remains unchanged. Flat directions
are directions along which the loss changes little or not at all. Every
behavior-preserving direction is necessarily first-order flat with respect to
any loss that depends only on predictive behavior, although the converse need
not hold.

Let $\mathcal L:\Theta\to\mathbb R$ denote a differentiable objective that
depends on $\theta$ only through the induced predictive behavior (for
example, population risk under the assumptions of
Section~\ref{sec:setup}).

\begin{proposition}[Behavior-Preserving Directions are Flat]
\label{prop:tangent-flat}

Let $v$ be tangent to a fiber $F_k$ at a point $\theta$. Assume that there
exists a differentiable curve
$\gamma:(-\varepsilon,\varepsilon)\to F_k$
with $\gamma(0)=\theta$ and $\gamma'(0)=v$, and that
$\mathcal L$ is differentiable. Then
\[
\nabla\mathcal L(\theta)^\top v=0.
\]
\end{proposition}

\begin{proof}

Let
$\gamma:(-\varepsilon,\varepsilon)\to F_k$
be the curve specified in the assumptions.
Since every point on $\gamma$ belongs to the same fiber,
the induced predictive behavior is constant along the curve.
Because $\mathcal L$ depends only on predictive behavior,
$\mathcal L(\gamma(t))$ is constant in $t$.

Differentiating at $t=0$ and applying the chain rule gives

\[
0
=
\frac{d}{dt}\mathcal L(\gamma(t))
\Big|_{t=0}
=
\nabla\mathcal L(\theta)^\top \gamma'(0)
=
\nabla\mathcal L(\theta)^\top v.
\]

\end{proof}

This proposition establishes a precise connection between behavioral
equivalence and flatness. Behavior-preserving directions correspond to exact
invariances of the predictive mapping and are therefore necessarily
first-order flat for any behavior-dependent objective. They provide an
idealized geometric model of realization multiplicity.

This result is intentionally weaker than the flat-minimum analyses commonly
used in deep learning
\citep{hochreiter1997flat,keskar2017sharp,jiang2020fantastic}. It establishes
only first-order invariance along behavioral fibers and does not require
assumptions about optimization trajectories, Hessian spectra, or local
curvature.

In practical models, approximate symmetries may produce nearly flat
directions even when exact behavioral invariance does not hold globally.
Conversely, a direction may be locally flat without corresponding to an exact
behavioral invariance. Thus, realization multiplicity and flatness are
closely related but not identical concepts.

\subsection{Geometric Interpretation of PAC-Bayes Z-Information}

The geometric structures discussed above influence PAC-Bayes analysis only
through their contribution to the realization-level term
\[
-\mathcal Z_{\mathrm{PB}}(Q\|P).
\]
Accordingly, PAC-Bayes Z-information admits a natural geometric
interpretation.

\begin{itemize}
\vspace*{-3pt}
\item In discrete settings, it measures posterior concentration relative to
finite symmetry orbits.
\vspace*{-3pt}
\item In continuous settings, it measures posterior concentration relative to
the conditional prior within behavior-preserving regions of parameter space.
\vspace*{-3pt}
\item In over-parameterized models, it quantifies the information-theoretic
consequences of realization multiplicity arising from symmetry, redundancy,
and behavior-preserving directions.
\end{itemize}

When the posterior and prior distribute mass similarly within behavioral
fibers,
$-\mathcal Z_{\mathrm{PB}}(Q\|P)$ is small.
When the posterior concentrates on a relatively small subset of realizations
relative to the conditional prior, the within-fiber divergence grows.
Geometrically, this corresponds to concentrating probability mass on a smaller
region of the realization set than that favored by the conditional prior.

Thus, PAC-Bayes Z-information does not measure uncertainty over predictive
behavior. Rather, it measures posterior concentration within behavioral
fibers relative to the conditional prior and thereby quantifies the
realization-level contribution to classical PAC-Bayes complexity.

\subsection{Summary}

Behavioral equivalence induces a geometric structure on parameter space in
which multiple realizations correspond to a single predictive behavior.

In discrete settings, realization multiplicity appears through symmetry
orbits. In continuous settings, it appears through the reference-measure size
of behavior-preserving regions and the existence of directions tangent to
behavioral fibers.

PAC-Bayes Z-information provides the information-theoretic counterpart of this
geometry by quantifying posterior concentration relative to the realization
structure encoded by behavioral fibers.

From this perspective, PAC-Bayes Z-information provides the
information-theoretic characterization of realization multiplicity, while the
geometry of behavioral fibers provides its structural interpretation. The
variational characterization of
Section~\ref{sec:pacbayes} shows that behavior-selection complexity is obtained
by removing within-fiber variation while preserving predictive behavior. The
geometric picture developed here clarifies what that removed variation
represents: symmetry-related realizations, behavior-preserving directions, and
large realization sets in continuous parameter spaces.

\section{Stability and Invariance from Realization Multiplicity}
\label{sec:stability}

The previous sections showed that behavioral equivalence separates predictive
behavior from internal realization. This section examines the consequences of
that separation for perturbations that change realizations while leaving
behavior unchanged. The resulting invariance properties clarify how realization
multiplicity manifests independently of both optimization dynamics and
training-set perturbations.

The notion considered here differs from classical algorithmic stability.
Algorithmic stability concerns the sensitivity of learned predictors to
perturbations of the training data
\citep{bousquet2002stability}. By contrast, we consider perturbations that
modify an internal realization while preserving predictive behavior.

\subsection{Fiber-Preserving Perturbations}

Let $F_k=\beta^{-1}(k)$ denote the fiber corresponding to behavior
$k\in\mathcal K$. We consider perturbations
$\theta\mapsto\theta'$ such that $\theta,\theta'\in F_k$.
Such perturbations move within a behavioral equivalence class and therefore
preserve predictive behavior by construction.

In over-parameterized models, these perturbations may arise from parameter
symmetries, redundant representations, or other forms of non-identifiability
\citep{neyshabur2017exploring,dinh2017sharp}. For example, permuting hidden
units together with their associated outgoing weights changes the parameter
vector but leaves the computed function unchanged. The resulting
configurations, therefore, remain in the same fiber.

\subsection{Loss Invariance on Fibers}

Because empirical and population risks depend only on predictive behavior,
they are invariant under fiber-preserving perturbations.

\begin{proposition}[Loss Invariance on Fibers]
\label{prop:loss-invariance}

For any behavior $k$ and any $\theta,\theta'\in F_k$,
\[
L(\theta)=L(\theta'),
\qquad
\hat L_S(\theta)=\hat L_S(\theta').
\]
\end{proposition}

\begin{proof}
Immediate from Lemma~\ref{lem:loss-invariance}.
\end{proof}

Thus, all configurations belonging to the same fiber have identical empirical
and population risks.

\subsection{Invariance Within Fibers}

Proposition~\ref{prop:loss-invariance} shows that perturbations that remain
within a fiber leave predictive behavior and risk unchanged. This invariance
is a structural consequence of behavioral equivalence.

Unlike algorithmic stability, it is not a property of a learning algorithm.
Rather, it is a property of the configuration space together with the behavior
map. It reflects the existence of multiple internal realizations that
implement the same predictive behavior.

Section~\ref{sec:geometry} showed that, under suitable regularity conditions,
directions tangent to fibers are first-order flat for any behavior-dependent
objective. The invariance described here may, therefore, be viewed as the
finite-perturbation counterpart of the infinitesimal invariances associated
with behavioral equivalence.

\subsection{Connection to PAC-Bayes Z-Information}

Section~\ref{sec:pacbayes} showed that PAC-Bayes complexity decomposes into a
behavior-selection term and a realization-level term. The realization-level term is

\[
-\mathcal Z_{\mathrm{PB}}(Q\|P)
=
\mathbb E_{k\sim\pi_Q}
\Big[
\mathrm{KL}
\big(
Q(\cdot\mid k)
\,\|\,P(\cdot\mid k)
\big)
\Big].
\]

This quantity measures how strongly the posterior concentrates within fibers
relative to the conditional prior. Since PAC-Bayes Z-information is the
negative expected within fiber KL divergence, it vanishes precisely when
the posterior and prior induce the same conditional distributions within fibers.
More generally, increasingly negative values of
$\mathcal Z_{\mathrm{PB}}(Q\|P)$ correspond to greater posterior
concentration relative to the conditional prior within fibers.

Importantly, PAC-Bayes Z-information is defined entirely through conditional
relative entropy and, therefore, does not depend on any geometric structure of
the fibers. The geometric interpretation developed in
Section~\ref{sec:geometry} provides intuition, but the quantity itself is
measure-theoretic.

As discussed in ~\cref{sec:flatness-zinfo}, it is important not to conflate this quantity with geometric flatness.
PAC-Bayes Z-information measures a distributional property: posterior
concentration relative to a reference distribution within fibers. Flatness, by
contrast, is a geometric property of the parameterization. Both arise from the
same realization structure, but they quantify different aspects of it.

\subsection{Interpretation}

Behavioral equivalence implies that many distinct configurations may realize
the same predictive behavior. This realization multiplicity has three
consequences developed throughout the paper.

First, it induces a decomposition of PAC-Bayes complexity into behavior-level
and realization-level terms. Second, in continuous parameter spaces and
under appropriate regularity assumptions, it gives rise to
behavior-preserving directions that are first-order flat for
behavior-dependent objectives. Third, it yields invariance of predictive
behavior and risk under fiber-preserving perturbations.

These observations suggest a useful conceptual distinction between
behavior selection complexity and realization multiplicity. Highly
over-parameterized models may exhibit substantial realization multiplicity
without a corresponding increase in behavior selection complexity. PAC-Bayes
Z-information isolates the contribution of this realization-level variation to
classical PAC-Bayes complexity, while the geometric and invariance viewpoints
provide complementary interpretations of the same underlying structure.

\subsection{Relation to Classical Stability}

Classical algorithmic stability analyzes how learned predictors change when
the training data are perturbed
\citep{bousquet2002stability}. The notion studied here concerns a different
source of variation:

\begin{itemize}
\vspace*{-6pt}
\item \textbf{Algorithmic stability:}
the sensitivity of the learning outcome to perturbations of the training set.
\vspace*{-16pt}
\item \textbf{Behavioral invariance:}
the invariance of predictive behavior under perturbations that remain within a
behavioral equivalence class.
\end{itemize}

The two notions address different aspects of learning systems and should be
viewed as complementary rather than competing perspectives.

\subsection{Summary}

Behavioral equivalence induces invariance under realization-preserving
perturbations. This invariance is distinct from classical algorithmic stability
and arises from realization multiplicity rather than insensitivity to
training data.

PAC-Bayes Z-information quantifies how the posterior mass is distributed within
behavioral equivalence classes and therefore provides an information-theoretic
measure of realization-level variation. Together with the geometric
perspective of Section~\ref{sec:geometry}, these results show that invariance,
fiber geometry, and realization-level PAC-Bayes complexity are complementary
manifestations of the same underlying realization structure induced by
behavioral equivalence.

\section{Related Work}
\label{sec:related}

\vspace*{-2pt}\paragraph{PAC-Bayes theory and complexity.}
PAC-Bayes theory provides distribution-dependent generalization bounds through
the tradeoff between empirical risk and the KL divergence between posterior and
prior distributions over hypotheses
\citep{mcallester1999pac,seeger2002pac,catoni2007pac}. Extensive work has
refined these bounds through localization, data-dependent priors, and improved
complexity control
\citep{kuzborskij2024better,jang2023tighter,rivasplata2020pacbayes,viallard2024}.
Connections with minimum description length (MDL) have likewise been widely
studied, with KL divergence admitting a coding-theoretic interpretation as
excess description length
\citep{rissanen1978model,grunwald2007minimum}.

PAC-Bayes theory is not restricted to parameter spaces. The underlying results
apply to arbitrary measurable hypothesis spaces, and several influential
analyses have been formulated directly in function space
\citep{seeger2002pac}. A major milestone was the demonstration of non-vacuous
PAC-Bayes bounds for deep neural networks by
\citet{dziugaite2017computing}, which stimulated extensive work on the
interpretation of PAC-Bayes complexity in over-parameterized models.

\vspace*{-4pt}\paragraph{Function-space, symmetry, and invariance perspectives.}
Several lines of work have emphasized that learning depends more directly on
induced predictors than on particular parameterizations. Examples include
Bayesian neural-network limits, Gaussian-process formulations, and neural
tangent kernel analyses
\citep{neal1995bayesian,lee2018deep,matthews2018gaussian,jacot2018ntk,
sun2019fvnns,dziugaite2021role}.

Related ideas arise in the study of invariance and symmetry. Neural networks
exhibit substantial parameter redundancies arising from permutations,
reparameterizations, and other transformations that preserve predictive
behavior. Exploiting such structure can reduce effective complexity and improve
generalization bounds
\citep{lyle2020benefits,behboodi2022pac}. More broadly, symmetry-aware and
invariance-aware learning can often be viewed as replacing a large
representation space with a smaller space of behaviorally distinguishable
predictors.

\vspace*{-4pt}\paragraph{Over-parameterization and flatness.}
The relationship between over-parameterization and generalization has been
widely studied through flat minima, sharpness, and loss geometry
\citep{hochreiter1997flat,keskar2017sharp}. Subsequent work demonstrated that
sharpness measures can be sensitive to parameterization and scaling
\citep{dinh2017sharp,jiang2020fantastic}. Parameter symmetries and
non-identifiability have likewise been studied as sources of equivalent
realizations
\citep{neyshabur2017exploring}. Our framework provides a complementary
perspective in which flatness is interpreted as one geometric manifestation of
realization multiplicity induced by behavioral equivalence.

\vspace*{-4pt}\paragraph{Information-theoretic perspectives.}
Shannon entropy and KL divergence quantify uncertainty over observable outcomes
and discrepancies between predictive distributions
\citep{shannon1948,shannon1951prediction,cover2006elements}, but do not
explicitly distinguish uncertainty over predictive behavior from multiplicity
among realizations that induce identical behavior. The terminology introduced
here is inspired by the microstate--macrostate distinction in statistical
physics and is conceptually related to recent work on zentropy
\citep{liu2022zentropy,liu2024zentropy}, although our development is entirely
learning-theoretic and PAC-Bayes in nature.

\vspace*{-4pt}\paragraph{Positioning of this work.}
The distinguishing feature of the present framework is an exact fiberwise
decomposition of the classical PAC-Bayes KL divergence induced by an arbitrary
measurable behavior map
$
\beta:\Theta\rightarrow\mathcal K.
$
The resulting decomposition separates uncertainty over predictive behavior
from variation among behaviorally equivalent realizations, identifies
PAC-Bayes Z-information as the realization-level contribution to classical
PAC-Bayes complexity, and yields a behavior-selection complexity term that is
exactly the minimum classical PAC-Bayes complexity among all posteriors
inducing the same distribution over predictive behaviors.

Viewed through this lens, prior work on function-space PAC-Bayes analysis,
invariance-aware learning, and symmetry-aware complexity control can be seen as
instances of a broader principle: learning complexity should depend on
distinguishable predictive behavior rather than arbitrary redundancy in its
internal realization. 

Existing approaches typically exploit particular sources of redundancy,
including group symmetries, parameter-space invariances, or specific
reparameterizations
\citep{lyle2020benefits,behboodi2022pac,
neyshabur2017exploring,dinh2017sharp}.
The present framework instead begins with an arbitrary measurable notion of
behavioral equivalence, from which the behavior–realization decomposition,
variational characterization, and behavior-level complexity arise uniformly.
\section{Summary and Discussion}
\label{sec:conclusion}

\subsection{Summary}

This paper showed that behavioral equivalence induces an exact decomposition
of the classical PAC-Bayes KL divergence into behavior-selection and
realization-level components. Within this decomposition, we introduced
PAC-Bayes Z-information as the negative expected conditional KL divergence,
thereby identifying the realization-level contribution to classical
PAC-Bayes complexity.

The key observation is that predictive risk depends only on induced predictive
behavior, whereas classical PAC-Bayes complexity is defined on the full
configuration space. By formalizing behavioral equivalence through a measurable
behavior map and applying measure disintegration, we established the identity
\[
\mathrm{KL}(Q\|P)
=
\mathrm{KL}(\pi_Q\|\pi_P)
-
\mathcal Z_{\mathrm{PB}}(Q\|P),
\]
where
\[
\mathcal Z_{\mathrm{PB}}(Q\|P)
=
-
\mathbb E_{k\sim\pi_Q}
\Big[
\mathrm{KL}
\big(
Q(\cdot\mid k)
\,\|\,P(\cdot\mid k)
\big)
\Big].
\]

This identity reveals that classical PAC-Bayes complexity consists of two
conceptually distinct contributions: uncertainty over predictive behavior and
variation among behaviorally equivalent realizations. PAC-Bayes Z-information
is precisely the negative expected within-fiber KL divergence and therefore
quantifies the realization-level contribution exactly.

The decomposition identifies the behavior-selection component,
$\mathrm{KL}(\pi_Q\|\pi_P)$, as the complexity associated with predictive
behavior itself. We further showed that this quantity admits an exact
variational characterization: it is the minimum classical PAC-Bayes complexity
among all posteriors inducing the same distribution over predictive behaviors.

Since
\[
\mathrm{KL}(\pi_Q\|\pi_P)
\le
\mathrm{KL}(Q\|P),
\]
the behavior-selection complexity is never larger than the classical
configuration-space complexity, and the gap is quantified exactly by
$-\mathcal Z_{\mathrm{PB}}(Q\|P)$.

Finally, Sections~\ref{sec:geometry} and~\ref{sec:stability} showed that the
same behavioral-equivalence structure also gives rise to geometric and
invariance-based interpretations through symmetry, realization multiplicity,
behavior-preserving directions, and fiber-preserving perturbations. Together,
these results provide complementary information-theoretic, geometric, and
structural perspectives on the role of realization multiplicity in classical
PAC-Bayes complexity.

\subsection{Discussion}

The central contribution of this work is the identification of an exact structural decomposition of PAC-Bayes complexity induced by behavioral equivalence. Although the underlying relative-entropy identity follows from the classical chain rule under disintegration, its interpretation through behavioral equivalence reveals a distinction between behavior-selection complexity and realization-level complexity that is not explicit in existing PAC-Bayes analyses. The novelty of the present framework lies in recognizing that,
when the measurable map is taken to be predictive behavior, the resulting
terms admit a natural learning-theoretic interpretation. Behavioral
equivalence reveals that the classical PAC-Bayes KL divergence naturally
separates into two conceptually distinct components:
(i) uncertainty over predictive behavior and
(ii) variation among realizations that implement the same behavior.

Only the first component directly affects predictive risk. The second reflects
how posterior mass is distributed among behaviorally equivalent realizations
and therefore contributes to configuration-space complexity without altering
prediction. The behavior--realization decomposition makes this distinction
explicit, identifies PAC-Bayes Z-information as an exact measure of the
realization-level contribution, and reveals behavior-selection complexity as
the irreducible component of classical PAC-Bayes complexity associated with
predictive behavior.

Viewed through this lens, the behavior-level PAC-Bayes formulation should not
be interpreted as an alternative to classical PAC-Bayes analysis, but as its
canonical representation on the behavior space. The corresponding complexity
term, $\mathrm{KL}(\pi_Q\|\pi_P)$, is precisely the minimum classical
PAC-Bayes complexity among all posteriors inducing the same distribution over
predictive behaviors. PAC-Bayes Z-information then quantifies exactly how much
additional complexity arises solely from redistributing posterior mass among
behaviorally equivalent realizations.

This perspective provides a principled way to distinguish complexity that is
intrinsic to predictive behavior from complexity that reflects only the choice
of realization. As a result, it offers a unified lens for interpreting
PAC-Bayes analyses of over-parameterized models, comparing alternative
parameterizations of the same predictor, and motivating future complexity
measures that operate directly on predictive behavior rather than internal
realization.

Beyond its implications for PAC-Bayes analysis, the framework provides a common
language for several phenomena that are often studied separately, including
parameter symmetries, realization multiplicity, behavior-preserving
perturbations, and geometric flatness. Rather than introducing these concepts
independently, behavioral equivalence reveals them as complementary
manifestations of the same underlying realization structure.

Several limitations should also be emphasized. First, the present analysis
relies on exact behavioral equivalence. Extending the framework to approximate
behavioral equivalence, where predictors induce nearly identical rather than
identical behaviors, remains an important direction for future work. Second,
the theory is structural rather than causal: it identifies realization-level
complexity exactly but does not establish a causal relationship between
realization multiplicity and improved generalization. Finally, PAC-Bayes
Z-information is a measure-theoretic quantity defined through conditional
relative entropy and should not be conflated with geometric quantities such
as flatness, although both arise naturally from the same fiber structure
induced by behavioral equivalence. Clarifying the precise relationships among
information-theoretic, geometric, and optimization-based notions of
realization multiplicity remains an important direction for future research.

\subsection{Future Work}

The framework suggests several directions for future work:

\begin{enumerate}
\item Developing computable estimators or practical proxies for PAC-Bayes
Z-information in modern neural networks.
\item Extending the theory to approximate behavioral equivalence, where
predictors are nearly indistinguishable rather than exactly identical.
\vspace*{-3pt}
\item Investigating whether optimization procedures implicitly favor regions of
high realization multiplicity.
\vspace*{-3pt}
\item Extending the behavior-realization decomposition to other
learning-theoretic frameworks, including stability-based,
information-theoretic, and compression-based analyses.
\item Empirically studying relationships among realization multiplicity,
generalization, robustness, calibration, and uncertainty estimation in
over-parameterized models.
\end{enumerate}

\subsection{Conclusion}

The behavior--realization decomposition shows that predictive behavior,
rather than internal realization, provides the natural level at which to
analyze PAC-Bayes complexity. This paper demonstrates that the distinction
between behavior-selection complexity and realization-level complexity arises
from the classical PAC-Bayes framework through an exact decomposition of
relative entropy together with its variational characterization. By making
this structure explicit, the resulting framework separates uncertainty over
predictive behavior from realization multiplicity and provides a unified
measure-theoretic interpretation of complexity, symmetry, and invariance in
over-parameterized learning systems. We expect this framework to provide a
foundation for future work on behavior-space learning theory, approximate
behavioral equivalence, invariance, and complexity measures for modern
over-parameterized models.

\section*{Acknowledgments}
This work was supported in part by grants from the National Science Foundation (NSF) and Penn State Clinical and Translational Science Institute (CTSI).

\newpage
\bibliography{zentropy_bibliography_from_main_complete}
\bibliographystyle{tmlr}

\newpage
\section*{Appendix}
\appendix
\section{Measure-Theoretic Preliminaries}
\label{app:preliminaries}

This appendix reviews the measure-theoretic concepts underlying the framework
developed in the main text, including behavioral equivalence, fibers,
pushforward measures, and disintegration. These constructions are standard in
probability and measure theory
\citep{kallenberg2002foundations,parthasarathy1967probability,
bogachev2007measure,klenke2013probability}. The purpose of this appendix is
not to develop new measure theory, but to establish notation and provide the
background needed for the behavior--realization decomposition used throughout
the paper.

\subsection{Behavioral Equivalence and Fibers}

Throughout the paper, predictive behavior is represented by a measurable
behavior map $\beta:\Theta\rightarrow\mathcal K$, where $\Theta$ is the
configuration space, and $\mathcal K$ is a space of predictive behaviors.

Behavioral equivalence is induced by this map:
\[
\theta\sim\theta'
\quad\Longleftrightarrow\quad
\beta(\theta)=\beta(\theta').
\]
Thus, two configurations are behaviorally equivalent precisely when they induce
the same predictive behavior.

For a behavior $k\in\mathcal K$, the associated fiber is

\[
F_k
=
\beta^{-1}(k)
=
\{\theta\in\Theta:\beta(\theta)=k\}.
\]

A fiber, therefore, contains every configuration that realizes the same
predictive behavior.

\paragraph{Example.}

A familiar instance of behavioral equivalence is provided by hidden
unit permutation symmetry in neural networks, where permuting exchangeable hidden
units leaves the input–output mapping unchanged. Such parameter
configurations, therefore, belong to the same behavioral fiber. A detailed
illustration appears in Section~\ref{sec:zinfo}, and the corresponding
realization-level computation is worked out in Appendix~\ref{app:worked:permutations}.

\subsection{Pushforward Measures and Behavior Distributions}

Let $Q\in\mathcal P(\Theta)$ be a probability distribution over
configurations.

The behavior map induces a probability distribution over behaviors through the
pushforward measure (also called the image measure)
\[
\pi_Q
=
\beta_\#Q
=
Q\circ\beta^{-1}.
\]
For any measurable subset $A\subseteq\mathcal K$,
\[
\pi_Q(A)
=
Q(\beta^{-1}(A)).
\]
Intuitively, $\pi_Q$ records how much probability mass $Q$ assigns to each
predictive behavior while ignoring how that mass is distributed among
equivalent realizations within a fiber.

The notation $\beta_\#Q$ is standard shorthand for the pushforward (or image)
measure induced by the measurable map $\beta$. Pushforward measures provide the
canonical mechanism by which probability distributions are transported through
measurable mappings
\citep{kallenberg2002foundations,bogachev2007measure}.

\subsection{Disintegration of Measures}

A central mathematical tool used throughout the paper is disintegration.

Disintegration is the measure-theoretic analog of conditioning and provides a
canonical decomposition of a probability measure relative to a measurable map
\citep{parthasarathy1967probability,kallenberg2002foundations,
bogachev2007measure,klenke2013probability}.

Assume that $(\Theta,\mathcal F_\Theta)$ is a standard Borel space and that
$\beta:\Theta\rightarrow\mathcal K$ is measurable. These are the same
assumptions introduced in Section~\ref{sec:setup}; they ensure the existence
of regular conditional distributions with respect to the behavior map.

Then standard disintegration theorems imply that every probability measure
$Q\in\mathcal P(\Theta)$ admits a regular conditional distribution with
respect to the behavior map $\beta$, supported on the fibers of $\beta$, and
therefore a decomposition of the form
\[
Q(d\theta)
=
Q(d\theta\mid k)\,
\pi_Q(dk).
\]
Here $Q(\cdot\mid k)$ is a conditional probability measure supported on the
fiber $F_k=\beta^{-1}(k)$.

Informally, disintegration separates uncertainty into two components:

\begin{enumerate}
\item uncertainty over predictive behaviors, represented by $\pi_Q$;

\item uncertainty over realizations conditional on a fixed predictive
behavior, represented by $Q(\cdot\mid k)$.
\end{enumerate}

This disintegration is the measure-theoretic foundation of the
behavior–realization decomposition developed in the main text.
Appendix~\ref{app:kl-decomposition} applies the classical chain rule for
relative entropy to this decomposition, with the measurable map taken to be the
behavior map $\beta$, yielding the behavior–realization decomposition of
classical PAC-Bayes complexity.

\subsection{Proofs of Behavioral Sufficiency Results}

\begin{proof}[Proof of Lemma~\ref{lem:loss-invariance}]

Behavioral equivalence implies
$\beta(\theta)=\beta(\theta')$, and therefore
$P_\theta(\cdot\mid x)=P_{\theta'}(\cdot\mid x)$ for every input $x$.

Since both population and empirical risks depend on a configuration only
through the predictive behavior represented by $\beta(\theta)$, it follows
immediately that
\[
L(\theta)=L(\theta'),
\qquad
\hat L_S(\theta)=\hat L_S(\theta').
\]
\end{proof}

\begin{proof}[Proof of Proposition~\ref{prop:sufficiency}]

By Lemma~\ref{lem:loss-invariance}, both population and empirical risks are
constant on fibers. Since $\beta$ is measurable and both risks are measurable
functions on $\Theta$, standard measurable-factorization results for functions
that are constant on the fibers of a measurable map
\citep{kallenberg2002foundations} imply the existence of measurable functions
$L_{\mathcal K}$ and $\hat L_{S,\mathcal K}$ on $\mathcal K$ such that
\[
L(\theta)=L_{\mathcal K}(\beta(\theta)),
\qquad
\hat L_S(\theta)=\hat L_{S,\mathcal K}(\beta(\theta)).
\]
Using the factorization above together with the pushforward measure,
\[
L(Q)
=
\mathbb E_{\theta\sim Q}
\!\left[
L_{\mathcal K}(\beta(\theta))
\right]
=
\mathbb E_{k\sim\pi_Q}
\!\left[
L_{\mathcal K}(k)
\right].
\]
If $\pi_Q=\pi_{Q'}$, the expectations coincide, yielding
\[
L(Q)=L(Q').
\]
Applying the same argument to $\hat L_{S,\mathcal K}$ gives
\[
\hat L_S(Q)=\hat L_S(Q').
\]
\end{proof}

\section{KL Decomposition and PAC-Bayes Z-Information}
\label{app:kl-decomposition}

This appendix provides proofs of the results in
Section~\ref{sec:zinfo}. The proofs rely on the classical chain rule for
relative entropy under measure disintegration, which decomposes relative
entropy into marginal and conditional contributions relative to a measurable
map
\citep{kallenberg2002foundations,parthasarathy1967probability,
bogachev2007measure}. The novelty of the present work lies not in this
measure-theoretic identity itself, but in its application to the behavior map
$\beta:\Theta\rightarrow\mathcal K$ introduced in
Section~\ref{sec:setup}, which yields the behavior–realization decomposition, its variational characterization, and the resulting behavior-level formulation of PAC-Bayes complexity.

Throughout, let $P,Q\in\mathcal P(\Theta)$ denote the prior and posterior
distributions, and let
$\pi_P=\beta_\#P$ and $\pi_Q=\beta_\#Q$
be the corresponding induced distributions on the behavior space.

Unless otherwise stated, we assume
\[
\mathrm{KL}(Q\|P)<\infty,
\]
which is the regime relevant to classical PAC-Bayes analysis. Under this
assumption, \(Q\ll P\), so the chain rule for relative entropy under
measure disintegration applies directly. The decomposition extends in the
usual extended-real sense when
\[
\mathrm{KL}(Q\|P)=\infty.
\]

Under the standard Borel assumptions of
Section~\ref{sec:setup}, both measures admit disintegrations
\[
P(d\theta)=P(d\theta\mid k)\,\pi_P(dk),
\qquad
Q(d\theta)=Q(d\theta\mid k)\,\pi_Q(dk).
\]

\subsection{Proof of Proposition~\ref{prop:kl-decomposition}}

\begin{proof}

Assume \[\mathrm{KL}(Q\|P)<\infty.\]
Then $Q\ll P$, and therefore
\[\pi_Q=\beta_\#Q\ll\beta_\#P=\pi_P.\]

Applying the chain rule for relative entropy with respect to the measurable
map $\beta$
\citep{kallenberg2002foundations,bogachev2007measure}
yields:
\[
\mathrm{KL}(Q\|P)
=
\mathrm{KL}(\pi_Q\|\pi_P)
+
\mathbb E_{k\sim\pi_Q}
\!\left[
\mathrm{KL}
\!\left(
Q(\cdot\mid k)
\,\middle\|\,
P(\cdot\mid k)
\right)
\right].
\]
\end{proof}

The nonnegativity of the conditional term immediately implies
Corollary~\ref{cor:dpi}, namely
\[
\mathrm{KL}(\pi_Q\|\pi_P)
\le
\mathrm{KL}(Q\|P).
\]
\subsection{Nonpositivity of PAC-Bayes Z-Information}

The statement following Definition~\ref{def:zinfo} follows directly from
the non-negativity of KL divergence.

\begin{proof}
For every behavior $k$,
\[
\mathrm{KL}
\!\left(
Q(\cdot\mid k)
\,\middle\|\,
P(\cdot\mid k)
\right)
\ge 0.
\]
Taking expectations preserves nonnegativity, so
\[
\mathcal Z_{\mathrm{PB}}(Q\|P)
=
-
\mathbb E_{k\sim\pi_Q}
\!\left[
\mathrm{KL}
\!\left(
Q(\cdot\mid k)
\,\middle\|\,
P(\cdot\mid k)
\right)
\right]
\le 0.
\]
Equality holds if and only if the conditional KL divergence vanishes
for $\pi_Q$-almost every $k$. By the standard characterization of
equality in relative entropy, this is equivalent to
\[
Q(\cdot\mid k)=P(\cdot\mid k)
\]
for $\pi_Q$-almost every $k$.
\end{proof}

\subsection{Proof of Proposition~\ref{prop:zinfo-identity}}

\begin{proof}
Rearranging the decomposition of
Proposition~\ref{prop:kl-decomposition} and substituting
Definition~\ref{def:zinfo} yields
\[
\mathcal Z_{\mathrm{PB}}(Q\|P)
=
\mathrm{KL}(\pi_Q\|\pi_P)
-
\mathrm{KL}(Q\|P).
\]
Or equivalently,
\[
\mathrm{KL}(\pi_Q\|\pi_P)
=
\mathrm{KL}(Q\|P)
+
\mathcal Z_{\mathrm{PB}}(Q\|P),
\]
which is the claimed identity.
\end{proof}

\subsection{Proof of Theorem~\ref{thm:optimal-representative}}

\begin{proof}
Recall that
\[
Q^\star(d\theta)
=
\int_{\mathcal K}
P(d\theta\mid k)\,
\pi_Q(dk).
\]
Because \(P(\cdot\mid k)\) is the measurable probability kernel obtained by
disintegrating \(P\) with respect to \(\beta\), the measure
$Q^\star$ is a well-defined probability measure on $\Theta$.
The existence of this measurable conditional kernel follows from the
standard Borel assumptions of Section~\ref{sec:setup} and the
disintegration theorem reviewed in
Appendix~\ref{app:preliminaries}.

For any measurable set $B\subseteq\mathcal K$,
\[
\pi_{Q^\star}(B)
=
Q^\star(\beta^{-1}(B))
=
\int_{\mathcal K}
P(\beta^{-1}(B)\mid k)\,
\pi_Q(dk).
\]
Because $P(\cdot\mid k)$ is supported on the fiber
$\beta^{-1}(k)$,
\[
P(\beta^{-1}(B)\mid k)
=
\mathbf 1_B(k).
\]
Therefore,
\[
\pi_{Q^\star}(B)
=
\int_{\mathcal K}
\mathbf 1_B(k)\,
\pi_Q(dk)
=
\pi_Q(B),
\]
showing that $\pi_{Q^\star}=\pi_Q$.

Applying Proposition~\ref{prop:kl-decomposition} to $Q^\star$ yields
\[
\mathrm{KL}(Q^\star\|P)
=
\mathrm{KL}(\pi_Q\|\pi_P),
\]
because
\[
Q^\star(d\theta)
=
P(d\theta\mid k)\,\pi_Q(dk),
\]
which is already a disintegration of $Q^\star$ with respect to the
behavior marginal $\pi_Q=\pi_{Q^\star}$.
Hence, a valid conditional distribution of $Q^\star$ is
\[
Q^\star(\cdot\mid k)
=
P(\cdot\mid k)
\]
for $\pi_Q$-almost every $k$. Consequently,
\[
\mathrm{KL}
\!\left(
Q^\star(\cdot\mid k)
\,\middle\|\,
P(\cdot\mid k)
\right)
=
0
\]
for $\pi_Q$-almost every $k$.

Now let $Q'$ satisfy $\pi_{Q'}=\pi_Q$.
Applying Proposition~\ref{prop:kl-decomposition} again gives
\[
\mathrm{KL}(Q'\|P)
=
\mathrm{KL}(\pi_Q\|\pi_P)
+
\mathbb E_{k\sim\pi_Q}
\!\left[
\mathrm{KL}
\!\left(
Q'(\cdot\mid k)
\,\middle\|\,
P(\cdot\mid k)
\right)
\right].
\]
The second term is nonnegative, so
\[
\mathrm{KL}(Q'\|P)
\ge
\mathrm{KL}(\pi_Q\|\pi_P).
\]
Since equality is achieved by $Q^\star$,
\[
\inf_{Q':\,\pi_{Q'}=\pi_Q}
\mathrm{KL}(Q'\|P)
=
\mathrm{KL}(\pi_Q\|\pi_P).
\]
\end{proof}

\subsection{Behavioral Sufficiency and Risk Preservation}

The variational characterization has learning-theoretic significance because
predictive risk depends only on behavior.

By Proposition~\ref{prop:sufficiency}, if two posteriors induce the same
behavior distribution, then they have identical population and empirical
risks. Since $\pi_{Q^\star}=\pi_Q$, Proposition~\ref{prop:sufficiency}
immediately yields
\[
L(Q^\star)=L(Q),
\qquad
\hat L_S(Q^\star)=\hat L_S(Q).
\]
Consequently, replacing $Q$ with $Q^\star$ preserves both empirical and
population performance while achieving the minimum classical PAC-Bayes
complexity among all posteriors inducing the same distribution over
predictive behaviors. This observation underlies the behavior-level PAC-Bayes interpretation
developed in Section~\ref{sec:pacbayes}.

\subsection{Interpretation}

Proposition~\ref{prop:kl-decomposition} specializes the classical chain rule
for relative entropy to the behavior map. This specialization separates
classical PAC-Bayes complexity into a behavior-selection term,
\[
\mathbb E_{k\sim\pi_Q}
\!\left[
\mathrm{KL}
\!\left(
Q(\cdot\mid k)
\,\middle\|\,
P(\cdot\mid k)
\right)
\right].
\]
This term measures the divergence between the posterior and prior
conditional distributions within behavioral fibers and therefore
quantifies realization-level variation after predictive behavior has
been fixed.

PAC-Bayes Z-information is the negative of the realization-level term.
Consequently, it quantifies the portion of classical PAC-Bayes
complexity attributable to variation among behaviorally equivalent
realizations.

Theorem~\ref{thm:optimal-representative} provides a complementary
variational interpretation. The quantity
\[
\mathrm{KL}(\pi_Q\|\pi_P)
=
\inf_{Q':\,\pi_{Q'}=\pi_Q}
\mathrm{KL}(Q'\|P)
\]
is the minimum classical PAC-Bayes complexity among all
configuration-space posteriors that induce the same behavior-level
distribution. Equivalently, it is the irreducible complexity associated
with the predictive behaviors encoded by $\pi_Q$ after all
realization-level variation has been removed.

The quantity
$-\mathcal Z_{\mathrm{PB}}(Q\|P)$ is precisely the optimality gap between the
classical PAC-Bayes complexity of $Q$ and the minimum complexity achievable
among all posteriors inducing the same behavioral distribution. It therefore quantifies exactly the excess classical PAC-Bayes complexity
attributable to realization-level variation within behavioral fibers.

\section{Behavior-Level PAC-Bayes Analysis}
\label{app:behavior-pacbayes}

This appendix provides proofs and supporting details for the behavior-level
PAC-Bayes analysis in Section~\ref{sec:pacbayes}. The behavior-level bound is
obtained by applying the classical PAC-Bayes theorem directly to the
measurable behavior space $\mathcal K$. The contribution is therefore not a
new concentration inequality but the identification of the corresponding
complexity term with the behavior-selection component of the exact
behavior–realization decomposition.

\subsection{Risk Depends Only on Behavior}

We first make explicit the risk identities used in the proof of the
Behavior-space formulation of the PAC-Bayes bound.

Let $
\pi_Q=\beta_\#Q
$
denote the pushforward (image) measure induced by the behavior map
$\beta$.

By Proposition~\ref{prop:sufficiency}, both population and empirical
risk depend on a posterior only through its induced behavior
distribution. Consequently,
\[
L(Q)=L(\pi_Q),
\qquad
\hat L_S(Q)=\hat L_S(\pi_Q).
\]
Equivalently, if two configuration-space posteriors satisfy
\[
\pi_Q=\pi_{Q'},
\]
then
\[
L(Q)=L(Q'),
\qquad
\hat L_S(Q)=\hat L_S(Q').
\]
A detailed proof of this factorization is given in
Appendix~\ref{app:preliminaries}.

These identities allow the behavior-space PAC-Bayes bound to be expressed
entirely in terms of the configuration-space posterior $Q$.

\subsection{Proof of Theorem~\ref{thm:behavior-pacbayes}}

\begin{proof}

By assumption, $\mathcal K$ is a standard Borel space. The prior $P$ on
$\Theta$ induces a prior
\[
\pi_P=\beta_\#P
\]
on $\mathcal K$.

Because $\mathcal K$ is a standard Borel space and $\pi_P$ is a
probability measure on $\mathcal K$, the classical PAC-Bayes theorem
applies directly on the behavior space $\mathcal K$
\citep{mcallester1999pac,seeger2002pac}. Thus, with probability
at least $1-\delta$ over the draw of $S\sim\mathcal D^n$,
simultaneously for all behavior-space posteriors
$\rho\in\mathcal P(\mathcal K)$,
\[
\mathrm{KL}
\!\left(
\hat L_S(\rho)
\,\middle\|\,
L(\rho)
\right)
\le
\frac1n
\left(
\mathrm{KL}(\rho\|\pi_P)
+
\log\frac{2\sqrt n}{\delta}
\right).
\]
Now let
\[
\rho=\pi_Q,
\]
where $Q$ is an arbitrary configuration-space posterior.

Substituting these identities into the PAC-Bayes bound yields
\[
\mathrm{KL}
\!\left(
\hat L_S(Q)
\,\middle\|\,
L(Q)
\right)
\le
\frac1n
\left(
\mathrm{KL}(\pi_Q\|\pi_P)
+
\log\frac{2\sqrt n}{\delta}
\right).
\]
Since $Q$ was arbitrary, the inequality holds simultaneously for all
configuration-space posteriors $Q$, establishing the theorem.

\end{proof}

\subsection{Proof of Proposition~\ref{prop:complexity-gap}}

\begin{proof}

By Proposition~\ref{prop:zinfo-identity},
\[
\mathcal Z_{\mathrm{PB}}(Q\|P)
=
\mathrm{KL}(\pi_Q\|\pi_P)
-
\mathrm{KL}(Q\|P).
\]
Rearranging gives
\[
\mathrm{KL}(Q\|P)
-
\mathrm{KL}(\pi_Q\|\pi_P)
=
-\mathcal Z_{\mathrm{PB}}(Q\|P).
\]
Since
\[
\mathcal Z_{\mathrm{PB}}(Q\|P)\le0,
\]
we have
\[
-\mathcal Z_{\mathrm{PB}}(Q\|P)\ge0.
\]
Therefore
\[
\mathrm{KL}(\pi_Q\|\pi_P)
\le
\mathrm{KL}(Q\|P).
\]
\end{proof}

\subsection{Proof of Corollary~\ref{cor:zinfo-pacbayes}}

\begin{proof}

Theorem~\ref{thm:behavior-pacbayes} gives
\[
\mathrm{KL}
\!\left(
\hat L_S(Q)
\,\middle\|\,
L(Q)
\right)
\le
\frac1n
\left(
\mathrm{KL}(\pi_Q\|\pi_P)
+
\log\frac{2\sqrt n}{\delta}
\right).
\]
Using Proposition~\ref{prop:zinfo-identity},
\[
\mathrm{KL}(\pi_Q\|\pi_P)
=
\mathrm{KL}(Q\|P)
+
\mathcal Z_{\mathrm{PB}}(Q\|P).
\]
Substituting this identity into the behavior-space PAC-Bayes bound gives
\[
\mathrm{KL}
\!\left(
\hat L_S(Q)
\,\middle\|\,
L(Q)
\right)
\le
\frac1n
\left(
\mathrm{KL}(Q\|P)
+
\mathcal Z_{\mathrm{PB}}(Q\|P)
+
\log\frac{2\sqrt n}{\delta}
\right).
\]
\end{proof}

\subsection{Interpretation}

The behavior-space PAC-Bayes bound is not a new concentration inequality.
Rather, it is the classical PAC-Bayes theorem applied to the measurable space
of predictive behaviors. Its significance lies in identifying the resulting
complexity term with the behavior-selection component of the exact
behavior--realization decomposition.

The resulting complexity term,
\[
\mathrm{KL}(\pi_Q\|\pi_P),
\]
depends only on uncertainty over predictive behavior. By
Theorem~\ref{thm:optimal-representative},
\[
\mathrm{KL}(\pi_Q\|\pi_P)
=
\inf_{Q':\,\pi_{Q'}=\pi_Q}
\mathrm{KL}(Q'\|P),
\]
so it is the minimum classical PAC-Bayes complexity among all
configuration-space posteriors that induce the same distribution over
predictive behaviors. By Theorem~\ref{thm:optimal-representative},
this minimum is attained by the fiber-symmetrized posterior
$Q^\star$, which preserves predictive behavior while eliminating
all realization-level divergence.

PAC-Bayes Z-information quantifies the exact gap between
configuration-space complexity and behavior-selection complexity. Since
\[
\mathrm{KL}(\pi_Q\|\pi_P)
=
\mathrm{KL}(Q\|P)
+
\mathcal Z_{\mathrm{PB}}(Q\|P),
\]
and
\[
\mathcal Z_{\mathrm{PB}}(Q\|P)\le0,
\]
Consequently, the behavior-space formulation isolates the irreducible
complexity associated with predictive behavior while removing complexity
arising solely from variation among behaviorally equivalent realizations,
without changing either empirical or population risk.

\section{Continuous-Space Interpretation of the Realization-Level KL Term}
\label{app:continuous-zinfo}

This appendix develops a continuous-space interpretation of the
realization-level KL divergence under additional assumptions on the
conditional prior and a chosen reference measure on each behavioral fiber.

The purpose of this appendix is to develop the geometric intuition for realization multiplicity discussed
in Section~\ref{sec:geometry}. The results in this appendix are not used in any proof of the main PAC-Bayes results. 
Under additional assumptions on the conditional
prior and a chosen reference measure on each behavioral fiber, the
realization-level KL divergence—and hence PAC-Bayes Z-information—admits an
occupancy-based interpretation relative to realization multiplicity.

The results of this appendix are supplementary and are not used in any proof
of the main PAC-Bayes results. The theory developed in
Sections~\ref{sec:setup}--\ref{sec:pacbayes} requires only the
measure-theoretic KL decomposition and remains valid independently of the
continuous-space interpretation presented here.

\subsection{Conditional KL Divergence on a Fiber}

Let $k\in\mathcal K$ and let
$F_k=\beta^{-1}(k)$ denote the corresponding behavioral fiber.

Assume that $F_k$ is equipped with a reference measure $\mu_k$, and that the
conditional prior and posterior admit densities
$p_k(\theta)$ and $q_k(\theta)$ with respect to $\mu_k$.

The choice of $\mu_k$ is not canonical and may depend on the application.
All densities and volume quantities in this appendix are understood relative
to the same fixed reference measure on the fiber.

The fiber-wise KL divergence is
\[
\mathrm{KL}
\big(
Q(\cdot\mid k)
\,\|\,P(\cdot\mid k)
\big)
=
\int_{F_k}
q_k(\theta)
\log
\frac{q_k(\theta)}{p_k(\theta)}
\,d\mu_k(\theta).
\]
This quantity measures posterior concentration relative to the conditional
prior within the behavioral fiber.

As emphasized in Section~\ref{sec:geometry}, the KL divergence itself does not
require any notion of geometric volume. The reference measure is introduced
only to provide intuition about realization multiplicity in continuous
configuration spaces.

\subsection{A Special Case: Uniform Conditional Priors}

Suppose that the conditional prior is uniform on the fiber,
$p_k(\theta)=1/\operatorname{Vol}(F_k)$,
where $\operatorname{Vol}(F_k)=\mu_k(F_k)$.

For simplicity, assume throughout this subsection that
$0<\operatorname{Vol}(F_k)<\infty$.

Substituting into the definition of KL divergence yields
\[
\begin{aligned}
\mathrm{KL}
\big(
Q(\cdot\mid k)
\,\|\,P(\cdot\mid k)
\big)
&=
\int_{F_k}
q_k(\theta)
\log
\frac{q_k(\theta)}
     {1/\operatorname{Vol}(F_k)}
\,d\mu_k(\theta)
\\
&=
\int_{F_k}
q_k(\theta)\log q_k(\theta)\,d\mu_k(\theta)
+
\log\operatorname{Vol}(F_k).
\end{aligned}
\]
Define the differential entropy of the posterior conditional distribution
on the fiber (relative to $\mu_k$) by
\[
H_k(Q)
=
-
\int_{F_k}
q_k(\theta)\log q_k(\theta)\,d\mu_k(\theta).
\]
As with differential entropy generally
\citep{cover2006elements}, the numerical value of $H_k(Q)$ depends on the
choice of reference measure. The KL divergence itself remains invariant
because both densities are defined relative to the same reference measure.

\begin{proposition}[Entropy-Volume Identity]
\label{prop:entropy-volume}

If the conditional prior is uniform on a fiber $F_k$, then
\[
\mathrm{KL}
\big(
Q(\cdot\mid k)
\,\|\,P(\cdot\mid k)
\big)
=
\log\operatorname{Vol}(F_k)
-
H_k(Q).
\]
\end{proposition}

\begin{proof}

Substituting
$p_k(\theta)=1/\operatorname{Vol}(F_k)$
into the definition of KL divergence gives
\[
\mathrm{KL}
\big(
Q(\cdot\mid k)
\,\|\,P(\cdot\mid k)
\big)
=
\int_{F_k}
q_k(\theta)\log q_k(\theta)\,d\mu_k(\theta)
+
\log\operatorname{Vol}(F_k).
\]
Using the definition of $H_k(Q)$ yields the result.
\end{proof}

The quantity $\operatorname{Vol}(F_k)$ is defined relative to the chosen
reference measure $\mu_k$ and, therefore, is not an intrinsic geometric
invariant of the fiber. The interpretation developed below concerns
realization multiplicity relative to the specified reference measure.

Accordingly, Proposition~\ref{prop:entropy-volume} should be viewed as a
special-case interpretation rather than a general characterization of the
realization-level KL divergence. For arbitrary conditional priors, the fiber-wise KL
divergence contains additional terms reflecting the variation of the conditional
prior within the fiber. The uniform-prior case isolates the contribution of
posterior occupancy relative to the available reference-measure size of the
fiber.

Although exact uniformity is rarely satisfied in practice, the resulting
interpretation remains informative whenever the conditional prior is
approximately uniform across a fiber. In such cases, the fiber-wise KL
divergence differs from the expression in
Proposition~\ref{prop:entropy-volume} by an additional term that quantifies
departures from uniform occupancy.

\subsection{Interpretation}

Proposition~\ref{prop:entropy-volume} separates the fiber-wise KL divergence
into two contributions:

\begin{enumerate}
\vspace*{-3pt}
\item the logarithm of the reference-measure size of the fiber,
$\log\operatorname{Vol}(F_k)$;
\vspace*{-3pt}
\item the entropy of posterior occupancy within that fiber,
$H_k(Q)$.
\end{enumerate}

The KL divergence is small when the posterior mass is distributed broadly
throughout the fiber and large when the posterior concentrates on a relatively
small subset of realizations.

Thus, under the uniform conditional-prior assumption, the realization-level
KL divergence can be interpreted as measuring posterior occupancy relative to
the available reference-measure size of a behavioral fiber.

Moreover, under the same assumption,
\[
\mathcal Z_{\mathrm{PB}}(Q\|P)
=
\mathbb E_{k\sim\pi_Q}
\Big[
H_k(Q)
-
\log\operatorname{Vol}(F_k)
\Big].
\]

Accordingly, PAC-Bayes Z-information may be interpreted as the expected
difference between posterior occupancy entropy and the logarithm of the
available reference-measure size of a behavioral fiber. This interpretation is
specific to the uniform-prior setting and serves only as geometric intuition
for the general measure-theoretic definition.

\subsection{Relation to Geometry and Scope}

The calculations in this appendix are standard consequences of relative
entropy and measure-theoretic probability
\citep{cover2006elements,csiszar1975divergence,amari2000methods,
kallenberg2002foundations}.

Section~\ref{sec:geometry} interprets realization multiplicity in terms of
symmetry, reference-measure size, and behavior-preserving directions. The
occupancy interpretation developed here provides an information-theoretic
counterpart to that geometric picture.

Relative to the chosen reference measure, large fibers correspond to
predictive behaviors that admit many realizations. In continuous settings, this realization multiplicity may arise through large
reference-measure size, high-dimensional families of behavior-preserving
directions, or other forms of geometric redundancy.

When the posterior mass remains broadly distributed throughout such fibers, the
realization-level KL term is small. Concentration onto a relatively small
subset of realizations increases the realization-level contribution and makes
PAC-Bayes Z-information more negative.

Under the assumptions above, PAC-Bayes Z-information measures occupancy
relative to available realization multiplicity rather than uncertainty over
predictive behavior itself.

This interpretation should be viewed as an intuition-building special case
rather than as part of the core theory. The behavior--realization
decomposition of Appendix~\ref{app:kl-decomposition} and the PAC-Bayes results
of Sections~\ref{sec:zinfo} and~\ref{sec:pacbayes} rely only on
measure-theoretic disintegration and relative entropy; they do not require
uniform conditional priors, finite-volume fibers, or any geometric notion of
reference measure. The assumptions introduced here serve only to provide a
more concrete geometric interpretation of realization multiplicity in
continuous parameter spaces.

\section{Geometry of Behavioral Fibers}
\label{app:geometry}

This appendix provides technical details supporting the geometric
interpretation of realization multiplicity presented in
Section~\ref{sec:geometry}.

The purpose is not to develop a full differential-geometric theory of
behavioral fibers. Rather, we establish several elementary geometric facts
that connect behavioral equivalence, symmetry, behavior-preserving
directions, and realization multiplicity.

Throughout, let
$\beta:\Theta\rightarrow\mathcal K$
denote the behavior map introduced in
Section~\ref{sec:setup}. For a behavior $k\in\mathcal K$, the associated
fiber is
\[
F_k=\beta^{-1}(k).
\]
The geometric terminology used in this appendix follows standard notions
of differentiable curves and tangent vectors
\citep{lee2018riemannian,docarmo1992riemannian}. No manifold structure on
fibers is assumed beyond what is required for the specific statements.

\subsection{Proof of Proposition~\ref{prop:symmetry-zinfo}}
\label{app:proof-symmetry-zinfo}

Recall Proposition~\ref{prop:symmetry-zinfo}.

\begin{proof}
Nonnegativity follows from the basic properties of KL divergence.

Suppose the fiber $F_k$ consists of a finite symmetry orbit of size $m$
and that the conditional prior is uniform on that orbit:
\[
P(\theta\mid k)=\frac1m.
\]
For any conditional posterior $Q(\cdot\mid k)$, using the convention
$0\log 0=0$,
\[
\mathrm{KL}
\!\left(
Q(\cdot\mid k)
\,\middle\|\,
P(\cdot\mid k)
\right)
=
\sum_{\theta\in F_k}
Q(\theta\mid k)
\log\!\bigl(m\,Q(\theta\mid k)\bigr).
\]
The minimum value is $0$, attained when
$Q(\cdot\mid k)=P(\cdot\mid k)$.

The maximum value is attained when the posterior concentrates all of its
mass on a single realization, in which case the divergence equals
$\log m$. Therefore
\[
0
\le
\mathrm{KL}
\!\left(
Q(\cdot\mid k)
\,\middle\|\,
P(\cdot\mid k)
\right)
\le
\log m.
\]
\end{proof}

The proposition shows that increasing the size of a symmetry orbit increases
the maximum possible realization-level contribution to classical PAC-Bayes
complexity while leaving predictive behavior unchanged.

The finite-orbit setting considered here may be viewed as the discrete
counterpart of the occupancy interpretation developed in
Appendix~\ref{app:continuous-zinfo}. In both cases, the conditional KL
term measures posterior concentration relative to the realization
structure associated with a fixed predictive behavior.

\subsection{Behavior-Preserving Curves}

Behavioral fibers may contain continuous families of realizations. When
the configuration space admits a differentiable structure, such families
can be represented by differentiable curves and their associated tangent
directions
\citep{lee2018riemannian,docarmo1992riemannian}.

Suppose
$\gamma:(-\varepsilon,\varepsilon)\to\Theta$
is a differentiable curve satisfying
$\gamma(t)\in F_k$
for all sufficiently small $t$.

Then every point on the curve induces the same predictive behavior.
Consequently, Proposition~\ref{prop:sufficiency} implies that both
population and empirical risk remain constant along the curve.

This observation provides the geometric counterpart of behavioral
equivalence: motion within a behavioral fiber changes the realization of
a predictor without changing the predictive behavior it induces.

\subsection{Proof of Proposition~\ref{prop:tangent-flat}}

Recall Proposition~\ref{prop:tangent-flat}.

\begin{proof}

Let
$\gamma(0)=\theta$
and
$\gamma'(0)=v$,
and suppose $\gamma(t)$ remains within a behavioral fiber.

Because predictive behavior is unchanged and
$\mathcal L$ depends on a configuration only through its predictive
behavior (Proposition~\ref{prop:sufficiency}),
$\mathcal L(\gamma(t))$
is constant.

Differentiating with respect to $t$ gives
\[
0
=
\frac{d}{dt}
\mathcal L(\gamma(t))
\Big|_{t=0}.
\]
Applying the chain rule yields
\[
\nabla\mathcal L(\theta)^\top\gamma'(0)
=
0.
\]
Since $\gamma'(0)=v$,
\[
\nabla\mathcal L(\theta)^\top v
=
0.
\]

\end{proof}

Thus, tangent directions to behavior-preserving curves preserve predictive
behavior to first order and are therefore necessarily first-order flat with
respect to any behavior-dependent objective.

The vector $v$ may therefore be interpreted as a behavior-preserving
direction at $\theta$. More precisely, $v$ is the tangent vector to a
curve that remains within a behavioral fiber and therefore preserves
predictive behavior to first order. The proposition shows that every
such tangent direction is necessarily first-order flat with respect to
any behavior-dependent objective.

This result is intentionally weaker than the flat-minimum analyses common
in deep learning
\citep{hochreiter1997flat,keskar2017sharp,jiang2020fantastic}. It
establishes only first-order invariance along behavioral fibers and does
not require assumptions about optimization trajectories, Hessian spectra,
or local curvature.

\subsection{Continuous Realization Multiplicity}

In discrete settings, realization multiplicity appears through finite
symmetry orbits.

In continuous settings, realization multiplicity may arise through
continuous families of equivalent realizations. Examples include
permutation symmetries, scaling symmetries, redundant representations,
and other forms of non-identifiability in neural networks
\citep{neyshabur2017exploring,dinh2017sharp}.

Behavioral equivalence therefore decomposes parameter-space variation into
two components: variation across behavioral fibers, which changes predictive
behavior, and variation within fibers, which changes only the realization of
that behavior. The latter corresponds precisely to realization multiplicity.
This geometric distinction mirrors the behavior--realization decomposition of
Sections~\ref{sec:zinfo} and~\ref{sec:pacbayes}.

From this perspective, the behavior map partitions parameter space into
behavioral fibers, each containing all realizations of a fixed predictive
behavior.

When many independent behavior-preserving directions exist, a single
predictive behavior may be implemented by a large family of
configurations. Geometrically, realization multiplicity may therefore be
associated with multiple independent behavior-preserving directions or,
relative to a chosen reference measure, with fibers that support many
realizations of the same predictive behavior.

Appendix~\ref{app:continuous-zinfo} shows that, under additional
assumptions on conditional priors and reference measures, the
realization-level KL term admits a corresponding occupancy-based
interpretation. Complementing this geometric perspective, the
variational characterization of
Theorem~\ref{thm:optimal-representative} shows that behavior-selection
complexity is obtained by removing variation within fibers while
preserving the induced distribution over predictive behaviors.

\subsection{Summary}

Behavioral fibers admit both discrete and continuous forms of realization
multiplicity.

Discrete multiplicity appears through symmetry orbits. Continuous
multiplicity may appear through behavior-preserving directions and
continuous families of equivalent realizations.

These structures do not alter predictive behavior, but they influence
how probability mass may be distributed among equivalent realizations.
Consequently, they provide geometric intuition for the realization-level
KL divergence that appears in the behavior--realization decomposition of
Appendix~\ref{app:kl-decomposition}. In particular, behavior-preserving
directions correspond to variations that leave predictive behavior
unchanged while contributing to the realization multiplicity quantified
by the conditional KL divergence and, equivalently, by PAC-Bayes
Z-information.

Together with Appendix~\ref{app:continuous-zinfo}, this geometric
perspective complements the measure-theoretic framework developed in the
main text by relating realization multiplicity to symmetry, occupancy,
and behavior-preserving directions in over-parameterized models. The
geometric interpretation is supplementary: the behavior--realization
decomposition and the resulting PAC-Bayes analysis rely only on the
measure-theoretic construction developed in
Appendix~\ref{app:kl-decomposition}.

\section{Behavioral Invariance and Fiber-Preserving Perturbations}
\label{app:stability}

This appendix provides technical details supporting the invariance results of
Section~\ref{sec:stability}.

The key observation is that behavioral equivalence separates perturbations
that alter predictive behavior from perturbations that merely change the
realization of that behavior. The latter induces a notion of behavioral
invariance that is distinct from classical algorithmic stability
\citep{bousquet2002stability,hardt2016train}.

\subsection{Proof of Proposition~\ref{prop:loss-invariance}}

Recall Proposition~\ref{prop:loss-invariance}.

\begin{proof}

Let $\theta,\theta'\in F_k$. By the definition of the fiber,
\[
\beta(\theta)=\beta(\theta')=k.
\]
Proposition~\ref{prop:sufficiency} shows that both population and
empirical risks depend on a configuration only through its induced
predictive behavior. Therefore,
\[
L(\theta)=L(\theta'),
\qquad
\hat L_S(\theta)=\hat L_S(\theta').
\]
\end{proof}

Thus, perturbations that remain within a behavioral fiber leave both
population and empirical risk unchanged.

\subsection{Behavioral Decomposition of Perturbations}

The fibers of the behavior map form a partition of the configuration
space.

Consequently, any perturbation of a configuration either

\begin{enumerate}
\item moves the configuration to a different fiber and thereby changes
predictive behavior; or

\item remains within the same fiber and therefore preserves predictive
behavior.
\end{enumerate}

This elementary observation underlies the distinction between
behavior-changing and fiber-preserving perturbations developed in
Section~\ref{sec:stability}. It also reflects the same
behavior--realization distinction that underlies the information-theoretic
decomposition developed in Sections~\ref{sec:zinfo}
and~\ref{sec:pacbayes}.

\subsection{Connection to Behavior-Preserving Directions}

Section~\ref{sec:geometry} and Appendix~\ref{app:geometry} established
that, under suitable regularity assumptions, tangent directions to
behavior-preserving curves are first-order flat directions of a
behavior-dependent objective
\citep{hochreiter1997flat,keskar2017sharp,dinh2017sharp}.

Specifically, if
\[
\gamma:(-\varepsilon,\varepsilon)\to\Theta
\]
is a differentiable curve contained in a behavioral fiber with
$\gamma(0)=\theta$ and $\gamma'(0)=v$,
then Proposition~\ref{prop:tangent-flat} implies that
\[
\nabla\mathcal L(\theta)^\top v = 0.
\]
Thus, fiber-preserving perturbations represented by differentiable curves
have zero first-order effect on any behavior-dependent objective.

This provides a geometric interpretation of the invariance associated
with behavioral equivalence. Under the regularity assumptions of
Proposition~\ref{prop:tangent-flat}, such directions arise as tangent
directions to behavior-preserving curves and represent first-order
variation among realizations that leaves predictive behavior unchanged.

\subsection{Behavioral Invariance versus Algorithmic Stability}

The invariance studied in this paper differs fundamentally from classical
algorithmic stability.

Algorithmic stability studies the sensitivity of learned predictors to
perturbations of the training dataset
\citep{bousquet2002stability,hardt2016train}.

By contrast, behavioral invariance concerns perturbations of a
configuration that remain within a behavioral fiber and therefore preserve
predictive behavior by construction.

The two notions address different sources of variation:

\begin{enumerate}
\item algorithmic stability concerns perturbations of the training data;

\item behavioral invariance concerns perturbations of realizations that
preserve predictive behavior.
\end{enumerate}

They should therefore be viewed as complementary rather than competing
perspectives.
\subsection{Interpretation}

Behavioral equivalence implies that multiple distinct configurations may
realize the same predictive behavior.

Whenever multiple realizations implement a fixed predictive behavior,
perturbations within a behavioral fiber leave both predictive behavior
and risk unchanged. Such realization multiplicity is closely related to
the parameter-space symmetries and non-identifiability phenomena
discussed in
\citep{neyshabur2017exploring,dinh2017sharp}.

The realization-level KL term and PAC-Bayes Z-information do not measure
behavioral invariance directly. Rather, they measure posterior
concentration relative to the conditional prior within behavioral
fibers. They are therefore distributional quantities defined through
conditional relative entropy and do not depend on any geometric
structure of the fibers. The geometric interpretation developed in
Section~\ref{sec:geometry} provides intuition, but the definitions
themselves are entirely measure-theoretic.

Viewed through the variational characterization of
Theorem~\ref{thm:optimal-representative}, behavioral invariance explains
why redistributing posterior mass within a behavioral fiber leaves
predictive behavior unchanged while potentially altering
configuration-space PAC-Bayes complexity.

Behavioral invariance, geometric interpretations based on symmetry and
behavior-preserving directions, and the realization-level KL term
therefore provide complementary geometric, variational, and
information-theoretic perspectives on the same realization structure
induced by behavioral equivalence.

\subsection{Summary}

Behavioral equivalence separates perturbations that alter predictive
behavior from perturbations that merely alter realization.

The latter leave predictive behavior and risk unchanged and, under
appropriate regularity assumptions, include tangent directions to
behavior-preserving curves within behavioral fibers.

The realization-level KL term and PAC-Bayes Z-information quantify
posterior concentration relative to the conditional prior within
behavioral fibers and therefore provide an information-theoretic
perspective on the same realization structure underlying behavioral
invariance.

Although behavioral invariance is distinct from generalization itself,
it helps explain why realization multiplicity appears naturally in the
behavior--realization decomposition: movement within behavioral fibers
changes realizations without changing predictive behavior.

\section{General Z-Information and Specialization to PAC-Bayes}
\label{app:general-zinfo}

This appendix places PAC-Bayes Z-information within a more general
measure-theoretic framework based on measurable maps, disintegration, and
relative entropy.

The purpose is conceptual rather than technical. It shows that the
behavior--realization decomposition developed in the main text is a
specialization of a general disintegration-based decomposition of relative
entropy. No results from this appendix are required for the development of the
main theory.

\subsection{A General Fiber Decomposition}

Let $\Theta$ be a measurable space equipped with probability measures
$P,Q\in\mathcal P(\Theta)$, and let
\[
r:\Theta\rightarrow\mathcal R
\]
be a measurable map into a measurable space $\mathcal R$.

For each $r\in\mathcal R$, the associated fiber is
\[
F_r
=
\{\theta\in\Theta:r(\theta)=r\}.
\]

The pushforward measures
\[
\pi_Q=Q\circ r^{-1},
\qquad
\pi_P=P\circ r^{-1}
\]
describe uncertainty over the image space $\mathcal R$.

Under the standard disintegration assumptions
\citep{kallenberg2002foundations,parthasarathy1967probability,
bogachev2007measure},
\[
Q(d\theta)
=
Q(d\theta\mid r)\,\pi_Q(dr),
\qquad
P(d\theta)
=
P(d\theta\mid r)\,\pi_P(dr).
\]

Applying the chain rule for relative entropy yields
\[
\mathrm{KL}(Q\|P)
=
\mathrm{KL}(\pi_Q\|\pi_P)
+
\mathbb E_{r\sim\pi_Q}
\!\left[
\mathrm{KL}
\bigl(
Q(\cdot\mid r)
\,\|\,P(\cdot\mid r)
\bigr)
\right].
\]

The second term measures divergence arising from the redistribution of
probability mass within fibers after the image-space distribution has
been fixed.

By the nonnegativity of KL divergence, the fiber-level term is always
nonnegative and vanishes if and only if
\[
Q(\cdot\mid r)=P(\cdot\mid r)
\]
for $\pi_Q$-almost every $r$.

\subsection{Specialization to Behavioral Equivalence}

The framework developed in the main text corresponds to the special case
in which the measurable map is the behavior map
\[
\beta:\Theta\rightarrow\mathcal K.
\]

The fibers
\[
F_k
=
\beta^{-1}(k)
\]
contain all configurations that induce the same predictive behavior.

In this setting, the negative fiber-level divergence becomes
\[
\mathcal Z_{\mathrm{PB}}(Q\|P)
=
-
\mathbb E_{k\sim\pi_Q}
\!\left[
\mathrm{KL}
\bigl(
Q(\cdot\mid k)
\,\|\,P(\cdot\mid k)
\bigr)
\right],
\]
which is precisely the PAC-Bayes Z-information introduced in
Section~\ref{sec:zinfo}.

Thus, PAC-Bayes Z-information is obtained by specializing the general
disintegration framework to the behavior map and taking the negative of
the resulting expected within-fiber KL divergence.

\subsection{Why the Behavioral Setting is Special}

Many measurable maps induce fiber decompositions, but the behavioral
setting possesses an additional property that is crucial for learning
theory.

By Proposition~\ref{prop:sufficiency},
\[
\beta(\theta)=\beta(\theta')
\quad\Longrightarrow\quad
L(\theta)=L(\theta'),
\qquad
\hat L_S(\theta)=\hat L_S(\theta').
\]

Consequently, both population and empirical risk factor through the
behavior map. Variation within a behavioral fiber does not directly
affect predictive behavior or risk.

This property gives the decomposition developed in
Sections~\ref{sec:zinfo} and~\ref{sec:pacbayes} its learning-theoretic
significance. The pushforward distribution over behaviors captures
behavior-level uncertainty relevant to prediction, whereas the
conditional distributions within fibers capture realization-level
variation among alternative realizations of the same predictive
behavior.
\subsection{Summary}

PAC-Bayes Z-information arises from a general disintegration-based
decomposition of relative entropy into image-space and fiber-level
components.

What distinguishes the behavioral setting is that predictive risk depends
only on behavior. This allows realization-level variation to be separated
from behavior-selection complexity and gives the within-fiber KL divergence a
direct learning-theoretic interpretation.

The PAC-Bayes framework developed in the main text is obtained by
specializing this general construction to the behavior map and
interpreting the resulting fibers as alternative realizations of the same
predictive behavior. In this specialization, the generic fiber-level
decomposition becomes the behavior--realization decomposition, the
image-space divergence becomes behavior-selection complexity, and the
within-fiber divergence becomes PAC-Bayes Z-information.

\section{Worked Examples and Explicit Computations}
\label{app:worked}

The preceding appendices establish the measure-theoretic foundations,
prove the main decomposition theorems, and develop their geometric and
information-theoretic consequences. This appendix serves a complementary purpose by presenting explicit
computations illustrating how the behavior--realization decomposition is
instantiated in concrete settings. The first example computes every quantity
appearing in the decomposition and illustrates the variational
characterization through fiber symmetrization. The second derives the
realization-level contribution associated with hidden-unit permutation
symmetry, complementing the illustrative discussion in the main text.

\subsection{A Finite Behavioral Quotient}
\label{app:worked:finite}

We begin with the simplest nontrivial example in which every quantity
appearing in the behavior--realization decomposition can be computed exactly.

Consider the configuration space
\[
\Theta
=
\{\theta_1,\theta_2,\theta_3,\theta_4\},
\]
and define the behavior map
\[
\beta(\theta_1)=\beta(\theta_2)=k_A,
\qquad
\beta(\theta_3)=\beta(\theta_4)=k_B.
\]

The behavioral fibers are therefore
\[
\beta^{-1}(k_A)
=
\{\theta_1,\theta_2\},
\qquad
\beta^{-1}(k_B)
=
\{\theta_3,\theta_4\}.
\]

Suppose the prior is
\[
P
=
\left(
\frac14,
\frac14,
\frac14,
\frac14
\right),
\]
and the posterior is
\[
Q
=
(0.6,0,0.4,0).
\]

Thus, the posterior assigns all of its probability mass to a single
configuration within each behavioral fiber.

\paragraph{Step 1: Behavior-Level Distributions.}

Applying the behavior map gives
\[
\pi_P=(0.5,0.5),
\qquad
\pi_Q=(0.6,0.4).
\]

Hence,
\[
\mathrm{KL}(\pi_Q\|\pi_P)
=
0.6\log\frac{0.6}{0.5}
+
0.4\log\frac{0.4}{0.5}
\approx
0.0201.
\]

\paragraph{Step 2: Fiber Conditional Distributions.}

Within each behavioral fiber,
\[
P(\cdot\mid k_A)
=
P(\cdot\mid k_B)
=
\left(
\frac12,\frac12
\right),
\]
whereas
\[
Q(\cdot\mid k_A)
=
Q(\cdot\mid k_B)
=
(1,0).
\]

Therefore,
\[
\mathrm{KL}
\!\left(
Q(\cdot\mid k_i)
\,\middle\|\,
P(\cdot\mid k_i)
\right)
=
\log2,
\qquad
i=A,B.
\]

Taking the expectation over
\(\pi_Q\)
gives
\[
\mathbb E_{k\sim\pi_Q}
\!\left[
\mathrm{KL}
\!\left(
Q(\cdot\mid k)
\,\middle\|\,
P(\cdot\mid k)
\right)
\right]
=
\log2
\approx
0.6931.
\]

Hence,
\[
\mathcal Z_{\mathrm{PB}}
=
-\log2.
\]

\paragraph{Step 3: The KL Decomposition.}

The classical divergence is
\[
\mathrm{KL}(Q\|P)
=
0.6\log\frac{0.6}{0.25}
+
0.4\log\frac{0.4}{0.25}
\approx
0.7132.
\]

Thus,
\[
\mathrm{KL}(Q\|P)
=
\mathrm{KL}(\pi_Q\|\pi_P)
-
\mathcal Z_{\mathrm{PB}},
\]
or numerically,
\[
0.7132
=
0.0201
+
0.6931.
\]

\paragraph{Step 4: Fiber Symmetrization.}

The fiber-symmetrized posterior is
\[
Q^\star
=
(0.3,0.3,0.2,0.2).
\]

Since
\[
\pi_{Q^\star}
=
\pi_Q,
\]
both posteriors induce identical predictive behavior.

Moreover,
\[
Q^\star(\cdot\mid k)
=
P(\cdot\mid k)
\]
for every behavioral fiber.

Consequently,
\[
\mathcal Z_{\mathrm{PB}}(Q^\star\|P)=0,
\]
and
\[
\mathrm{KL}(Q^\star\|P)
=
\mathrm{KL}(\pi_Q\|\pi_P)
\approx0.0201.
\]

The principal quantities appearing in the decomposition are summarized in
Table~\ref{tab:worked:finite}.

\begin{table}[h]
\centering
\caption{Summary of the behavior--realization decomposition for Example~\ref{app:worked:finite}.}
\label{tab:worked:finite}
\small
\begin{tabular}{lc}
\toprule
Quantity & Value\\
\midrule
$\mathrm{KL}(Q\|P)$ & $0.7132$\\
$\mathrm{KL}(\pi_Q\|\pi_P)$ & $0.0201$\\
$-\mathcal Z_{\mathrm{PB}}(Q\|P)$ & $0.6931$\\
$\mathrm{KL}(Q^\star\|P)$ & $0.0201$\\
\bottomrule
\end{tabular}
\end{table}

\paragraph{Interpretation.}

This example illustrates every component of the
behavior--realization decomposition. Although the posterior $Q$ and its
fiber-symmetrized representative $Q^\star$ induce identical predictive
behavior, their classical PAC-Bayes complexities differ by more than a factor
of thirty. In this example, over $97\%$ of the classical PAC-Bayes complexity
arises from realization-level divergence rather than uncertainty over
predictive behavior. Fiber symmetrization removes this realization-level
contribution while preserving predictive behavior, thereby realizing the
variational characterization established in
Theorem~\ref{thm:optimal-representative}.

\subsection{Explicit Computation for Hidden-Unit Permutation Symmetry}
\label{app:worked:permutations}

Example~\ref{app:worked:finite} computed every component of the behavior–realization decomposition in a finite setting. We now isolate the realization-level term for the hidden-unit permutation symmetry discussed in the main text.

Consider a behavioral fiber generated by all permutations of $h$
exchangeable hidden units. Under the idealized assumptions described in
Section~\ref{sec:zinfo}, suppose the conditional prior assigns equal
probability to every realization within the fiber,
\[
P(\cdot\mid k)
=
\left(
\frac1{h!},
\ldots,
\frac1{h!}
\right),
\]
while the conditional posterior concentrates on one representative
realization,
\[
Q(\cdot\mid k)
=
(1,0,\ldots,0).
\]

The conditional KL divergence within the fiber is therefore
\[
\begin{aligned}
\mathrm{KL}
\!\left(
Q(\cdot\mid k)
\,\middle\|\,
P(\cdot\mid k)
\right)
&=
\sum_{i=1}^{h!}
Q_i
\log
\frac{Q_i}{P_i}
\\
&=
\log(h!).
\end{aligned}
\]

If the same conditional structure holds for
$\pi_Q$-almost every behavioral fiber, then
\[
-\mathcal Z_{\mathrm{PB}}
=
\mathbb E_{k\sim\pi_Q}
\!\left[
\mathrm{KL}
\!\left(
Q(\cdot\mid k)
\,\middle\|\,
P(\cdot\mid k)
\right)
\right]
=
\log(h!),
\]
and consequently
\[
\mathrm{KL}(Q\|P)
=
\mathrm{KL}(\pi_Q\|\pi_P)
+
\log(h!).
\]

Using Stirling's approximation,
\[
\log(h!)
=
h\log h
-
h
+
O(\log h).
\]

For illustration,
\[
\log(10!)
\approx
15.10,
\qquad
\log(100!)
\approx
363.74.
\]

\paragraph{Interpretation.}

This computation makes explicit the approximation summarized in the main
text. The additional $\log(h!)$ term arises entirely from concentration of
posterior probability within a behavioral fiber and does not reflect greater
uncertainty over predictive behavior. As the degree of symmetry increases,
realization multiplicity can therefore contribute substantially to the
classical PAC-Bayes KL divergence even when predictive behavior remains
unchanged.

\end{document}